\documentclass{article} 
\usepackage{times}  
\usepackage{url}  
\usepackage{graphicx} 
\usepackage{xspace}
\usepackage{amssymb,amsthm,amsmath}
\usepackage{misc}
\usepackage{color}
\usepackage{centernot} 
 \newlength\titlebox 
 \twocolumn 

\newcommand{\Aa}{{\cal{A}}}

\newcommand{\Ff}{{\cal{F}}}
\newcommand{\Ii}{{\cal{I}}}
\newcommand{\Jj}{{\cal{J}}}
\newcommand{\Kk}{{\cal{K}}}
\newcommand{\Ll}{{\cal{L}}}
\newcommand{\Mm}{{\cal{M}}}

\newcommand{\Oo}{{\cal{O}}}

\newcommand{\Qq}{{\cal{Q}}}

\newcommand{\notmodels}{\centernot\models}
\newcommand{\fentails}{\models_{\mathsf{fin}}}
\newcommand{\notfentails}{\notmodels_{\!\!\mathsf{fin}}}

\newcommand{\ALCOI}{\ensuremath{{\cal{ALC\hspace{-0.25ex}O\hspace{-0.06ex}I}}}\xspace}
\renewcommand{\ALCO}{\ensuremath{{\cal{ALC\hspace{-0.25ex}O}}}\xspace}
\renewcommand{\ALCIF}{\ensuremath{{\cal{ALCI\hspace{-0.1ex}F}}}\xspace}
\renewcommand{\ALCHOIb}{\ensuremath{{\cal{ALCHOI}\textit{b}}}\xspace}

\newcommand{\SOF}{\ensuremath{{\cal{S\hspace{-0.2ex}O\hspace{-0.24ex}F}}}\xspace}
\newcommand{\SOI}{\ensuremath{{\cal{S\hspace{-0.2ex}O\hspace{-0.06ex}I}}}\xspace}
\newcommand{\SIF}{\ensuremath{{\cal{S\hspace{-0.06ex}I\hspace{-0.1ex}F}}}\xspace}

\newcommand{\SOIF}{\ensuremath{{\cal{SOIF}}}\xspace}

\newcommand{\K}{\ensuremath{{\cal{K}}}\xspace}

\newcommand{\var}{\textit{var}}
\newcommand{\Nomi}{\mn{Nom}}
\newcommand{\Ind}{\mn{Ind}}
\newcommand{\Rol}{\mn{Rol}}

\newcommand{\Tp}{\mn{Tp}}
\newcommand{\CN}{\mn{CN}}
\newcommand{\tp}{\mn{tp}}
\newcommand{\tra}{\textsc{tr}}
\newcommand{\ntra}{\textsc{nt}}
\newcommand{\twoexp}{\ensuremath{\textsc{2ExpTime}}\xspace}

\newtheorem{definition}{Definition}
\newtheorem{theorem}{Theorem}

\newtheorem{corollary}{Corollary}
\newtheorem{lemma}{Lemma}

\begin{document}
\title{Finite Query Answering in Expressive Description Logics with Transitive Roles}
\author{ Tomasz Gogacz \\ University of Warsaw, Poland\\
t.gogacz@mimuw.edu.pl
\and
Yazm\'in Ib\'a\~nez-Garc\'ia \\
TU Wien, Austria \\ yazmin.garcia@tuwien.ac.at \\
\and Filip Murlak \\ University of Warsaw, Poland \\
 fmurlak@mimuw.edu.pl
}
\maketitle
\begin{abstract}
We study the problem of \emph{finite} ontology mediated query answering
(FOMQA), the variant of OMQA where the represented world is assumed to
be finite, and thus only finite models of the ontology are considered.
We adopt the most typical setting with unions of
conjunctive queries and ontologies expressed in description logics
(DLs). The study of FOMQA is relevant in settings that are not 
finitely controllable. 
This is the case not only for DLs without the finite model property,
but also for those allowing transitive role declarations. When
transitive roles are allowed, evaluating queries is challenging: FOMQA
is undecidable for  \SHOIF and only known to be decidable for the Horn
fragment of \ALCIF. We show decidability of FOMQA for three proper
fragments of \SOIF: \SOI, \SOF, and \SIF. Our approach is to
characterise models relevant for deciding finite query
entailment. Relying on a certain regularity of these models, we
develop automata-based decision procedures with optimal complexity bounds.
\end{abstract}

\section{Introduction}
Evaluating queries in the presence of background knowledge has been
extensively studied in several communities. A particularly prominent
take on this problem is ontology mediated query answering (OMQA) where
background knowledge represented by an ontology is leveraged
to infer more complete answers to queries~\cite{BienvenuO15}. A widely accepted family of ontology languages with varying expressive power is offered by Description Logics (DLs)~\cite{Baader2010DLH}, while the most commonly studied query language is that of (unions of) conjunctive queries.



Often, the intended models of the ontology are finite and this additional assumption allows to infer more properties: \emph{ finite ontology mediated query answering (FOMQA)} is the variant of OMQA restricted to finite models. For some logics the finite variant and the unrestricted variant of the problem coincide; we then say that OMQA is \emph{finitely controllable}. Studying FOMQA is interesting in settings lacking finite controllability. This is the case not only for DLs lacking the \emph{finite model property} (e.g., DLs allowing both inverse roles and number restrictions), but also for logics allowing \emph{transitive role declarations}.  Indeed, it has been recently proved that FOMQA is undecidable for \SHOIF ontologies~\cite{Rudolph16}, whereas  the only fragment known to be decidable is Horn-\ALCIF~\cite{GarciaLS14}; more expressive fragments of \SHOIF are entirely uncharted. In this paper, we establish decidability for three of them: \SOI, \SOF, and \SIF.

OMQA is closely related to \emph{query answering under integrity constraints} in database theory: given a finite database instance and a set of constraints, 
determine answers to a query that are certain to hold over any extension of the given instance that satisfies the constraints. Among important classes of constraints are inclusion dependencies (IDs) and functional dependencies (FDs).  
This problem, often called \emph{open-world query answering (OWQA)}, has also been studied in the variant considering only finite extensions of the given database instance (finite OWQA), which is directly relevant for our work. 
OWQA over IDs is known to be finitely controllable ~\cite{JohnsonK84,Rosati11}. Rosati's techniques were extended to show finite controllability for the guarded fragment of first order logic~\cite{BaranyGO13}. Under combinations of IDs and FDs, OWQA is undecidable, both unrestricted and finite, but multiple decidable fragments have been isolated. For instance, for \emph{non-conflicting} IDs and  FDs~\cite{CaliLR03}, unrestricted OWQA is decidable. However, finite OWQA is undecidable already for non-conflicting IDs and keys, which are less expressive than FDs~\cite{Rosati11}. The work of \cite{AmarilliB15} investigates finite OWQA for unary IDs and  FDs over arbitrary signatures.

Combinations of unary IDs and unary FDs can be expressed in relatively simple DLs. This relationship and the techniques developed by~\cite{CosmadakisKV90} have been exploited in the study of finite satisfiability for simple DLs~\cite{Rosati08}. Indeed, finite satisfiability has been studied extensively~\cite{Calvanese96,LutzST05,Kazakov08,Pratt-Hartmann0_gc2_sat}, but  FOMQA has received limited attention in the DL community. The mentioned results on the guarded fragment give finite controllability for DLs up to \ALCHOIb. For non-finitely-controllable DLs, only the already mentioned results about \SHOIF and Horn-\ALCIF are known. For Datalog$^\pm$, finite controllability holds for  several fragments~\cite{GogaczM17,AmendolaLM17,BagetLMS11,CiviliR12}.
Finally, \cite{Pratt-Hartmann_datacomplexity} studies finite query answering for expressive fragments of first order logic and establishes undecidability for the two variable fragment with counting quantifiers ($\Cmc^2$), and decidability  for its guarded fragment, $\mathcal{GC}^2$. Decidability of $\mathcal{GC}^2$ has no direct implications for DLs with nominals or transitive roles, but it proves useful in the study of $\SIF$.  


\paragraph{Contributions.} We show that the combined complexity of FOMQA is in \twoexp for \SOI, \SOF and \SIF. These bounds are tight by existing matching lower bounds for OMQA for less expressive logics enjoying finite controllability~\cite{NgoOS16,Lutz08}. We present a direct construction of finite counter-models from arbitrary tree-like counter models for \ALCOI,  thus re-proving finite controllability. An extension of this construction builds finite counter-models from special tree-like models of \SOI and \SOF, which are guaranteed to exist whenever finite counter-models exist. This way finite query entailment reduces to entailment over a certain class of tree-like models recognisable by tree automata. For \SIF, we show that to some extent one can separate the reasoning about transitive and non-transitive (possibly functional) roles, and design a procedure that uses the decidability results for $\SOI$ and $\ALCIF$ as black boxes. The latter is derived from the work of~\cite{Pratt-Hartmann_datacomplexity}.



\section{Preliminaries}

The DL \SOIF extends the classical DL \ALC
with transitivity declarations on roles ($\Smc$), nominals ($\Omc$),
inverses ($\Ii$), and role functionality declarations ($\Fmc$)~\cite{Baader2010DLH}. We
assume a signature of countably infinite disjoint sets of
\emph{concept names} $\mn{N_C} = \{A_1, A_2, \dots \}$, \emph{role
names} $\mn{N_R}= \{r_1, r_2, \dots \} $ and \emph{individual names}
$\mn{N_I} = \{a_1, a_2, \dots \}$.  \emph{\SOIF-concepts $C,D$} are
defined by the grammar:
$$ C,D ::= \top \mid A \mid \neg C \mid C \sqcap D \mid \{a\}\mid \exists r. C\,, $$
where $r\in \mn{N_R} \cup \{r^- \mid r \in \mn{N_R}\}$ is a
\emph{role}.  Roles of the form $r^-$ are called \emph{inverse roles}. 
A \emph{\SOIF TBox} \Tmc is a finite set of \emph{concept inclusions
(CIs)}  $C\sqsubseteq D$, {\em transitivity declarations}
$\mn{Tr}(r)$, {\em functionality declarations} $\mn{Fn}(r)$, where
$C,D$ are \SOIF-concepts and $r$ is a role.
We assume that if the TBox contains  $\mn{Tr}(r)$, then
it contains neither $\mn{Fn}(r)$ nor $\mn{Fn}(r^-)$.
With an appropriate extension of the signature, each \SOIF TBox can be
transformed into an equivalent TBox whose each CI has one of the
following normal forms:
\[ \bigsqcap A_i \sqsubseteq \bigsqcup B_j\,, \quad A \equiv \{a\} \,,\quad 
   A \sqsubseteq \forall r. B \,,\quad A \sqsubseteq \exists r. B \,,\]
where empty conjunction is equivalent to $\top$ and empty disjunction
to $\bot$.  We also assume that for each concept name $A$ used in
$\Tmc$ there is a \emph{complementary} concept name $\bar A$
axiomatised with CIs $\top \sqsubseteq A \sqcup \bar A$ and 
$A \sqcap \bar A \sqsubseteq \bot$.

\SOI, \SOF and \SIF TBoxes are restrictions of \SOIF  TBoxes. \SOI
TBoxes  do not contain functionality declarations, whereas concept
inclusions in \SOF and \SIF do not contain inverse roles and nominals,
respectively.  Because the inverse of a transitive role is transitive anyway,
for \SOI, \SIF, and \SOIF  we shall assume that if $\mn{Tr}(r)$ is
present in the TBox, then so is $\mn{Tr}(r^-)$.  
 
An {\em ABox} is a finite set of {\em concept} and {\em role}
assertions of the form $A(a)$ and $r(a,b)$, where $A \in \mn{N_C}$, $r
\in \mn{N_R}$ and $\{a,b\} \subseteq \mn{N_I}$.  A \emph{knowledge
base (KB)} is a pair $\Kk=(\Tmc, \Amc)$. We write
$|\Kk|$ for $|\Amc|+|\Tmc|$. We use $\CN(\Kk)$,
$\mn{Rol}(\Kk)$, $\Nomi(\Kk)$, and $\Ind(\Kk)$ to denote, respectively,
the set of \emph{all concept names, roles, nominals, and individuals
occurring in $\Kk$}. We stress that if $r$ occurs in $\Kk$,  but
$r^-$ does not, then $r^-\notin \mn{Rol}(\Kk)$. 

A \emph{unary type} is a subset of $\CN(\Kk)$ that contains exactly
one of the concept names $A$, $\bar A$ for each $A \in \CN(\Kk)$. We
write $\Tp(\Kk)$ for the set of all unary types.

\medskip

The semantics is defined via interpretations $\Ii = (\Delta^\Ii,
\cdot^\Ii)$ with a non-empty \emph{domain} $\Delta^\Ii$ and an
\emph{interpretation function} $\cdot^\Ii$ assigning to each $A \in
\CN(\Kk)$ a set $A^\Imc \subseteq \Delta^\Imc$ and to each role name
$r$ with $r\in\Rol(\Kk)$ or $r^-\in\Rol(\Kk)$, a binary relation
$r^\Imc \subseteq \Delta^\Imc \times \Delta^\Imc$. The interpretation
of complex concepts and roles is defined as
usual~\cite{Baader2010DLH}.  We only consider interpretations
complying with the \emph{standard name assumption} in 
the sense that $a^\Ii = a$ for every $a \in \mn{N_I}$.

An interpretation $\Ii$ \emph{satisfies} $\alpha \in \Tmc \cup \Amc$,
written as $\Ii \models \alpha$, if the following holds: if $\alpha$
is a CI $C\sqsubseteq D$ then $C^\Ii\subseteq D^\Ii$, if $\alpha$ is a
transitivity declaration $\mn{Tr}(r)$ then $r^{\Ii}$is transitive, if
$\alpha$ is a functionality declaration $\mn{Fn}(r)$ then $r^{\Ii}$ is
a partial function, if $\alpha$ is an assertion $A(a)$ then $a \in
A^{\Ii}$, and if $\alpha$ is an assertion $r(a,b)$ then $(a,b) \in
r^\Ii$. 
 
Finally,  $\Ii$ is a \emph{model} of:  
a TBox \Tmc, denoted $\Ii\models\Tmc$, if $\Ii\models \alpha$ for all
$\alpha\in\Tmc$;   
an ABox \Amc, denoted $\Ii\models\Amc$, if $\Ii\models \alpha$ for all
$\alpha\in\Amc$;  
and a KB $\Kk$ if $\Ii \models \Tmc$ and $\Ii \models \Amc$.   

Interpretation $\Ii$ is a \emph{subinterpretation} of interpretation
$\Jj$, written as $\Ii \subseteq \Jj$, if $\Delta^\Ii \subseteq
\Delta^\Jj$, $A^\Ii \subseteq A^\Jj$, and $r^\Ii \subseteq r^\Jj$ for
all $A \in \CN(\Kk)$, $r\in \mathsf{Rol}(\Kk)$.  An interpretation
$\Ii$ is a subinterpretation of $\Jj$ \emph{induced} by $\Delta_0
\subseteq \Delta^\Jj$, written as $\Ii = \Jj \upharpoonright {\Delta_0}$,  if
$\Delta^\Ii = \Delta_0$, $A^\Ii = A^\Jj \cap \Delta_0$, and $r^\Ii =
r^\Jj \cap \Delta_0 \times \Delta_0$ for all $A \in \CN(\Kk)$, $r\in
\mathsf{Rol}(\Kk)$.  We write $\Jj \setminus X$ for the
subinterpretation of $\Jj$ induced by $\Delta^\Jj \setminus X$. 

Let $\Ii$ and $\Jj$ be interpretations of $\Kk$.  A
\emph{homomorphism} from $\Ii$ to $\Jj$, written as $h : \Ii \to \Jj$
is a function $h : \Delta^\Ii \to \Delta^\Jj$ that preserves roles,
concepts, and individual names; that is, $(h(d), h(d')) \in r^\Jj$
whenever $(d,d') \in r^\Ii$, $r \in \mathsf{Rol}(\Kk)$, $h(d) \in
A^\Jj$ whenever $d \in A^\Ii$, $A \in \CN(\Kk)$, and $h(a)= a$ for all
$a \in \Ind(\Kmc)$.  Note that $\Ii \subseteq \Jj$ iff the
identity mapping $\mathrm{id}$ is a homomorphism $\mathrm{id} : \Ii
\to \Jj$.

\medskip

Let \mn{N_V} be a countably infinite set of \emph{variables}. An
\emph{atom} is an expression of the form $A(x)$ or $r(x,y)$ with $A
\in\mn{N_C}$, $r \in \mn{N_R}$, and $x,y \in \mn{N_V}$, referred to as
\emph{concept atoms} and \emph{role atoms}, respectively. A
\emph{conjunctive query (CQ)} $Q$ is an existentially quantified
conjunction $q$ of atoms, $\exists x_1 \cdots  \exists x_n \, q\,.$
For simplicity we restrict it to be Boolean; that is, $\var(Q) = \{x_1,
\dots , x_n\}$. This is without loss of generality since the case of
non-Boolean CQs can be reduced to the case of Boolean queries; see
e.g. \cite{RudolphG10}.

A \emph{match for $Q$ in \Imc} is a total function $\pi:\var(Q)\to
\Delta^\Ii$ such that $\Ii,\pi\models q$ under the standard semantics
of first-order logic. An interpretation $\Ii$ satisfies $Q$, written
as $\Ii\models Q$ if there exists a match for $Q$ in $\Ii$. Note that
we do not consider queries with constants (i.e., individual names);
such queries can be viewed as non-boolean queries with a fixed
valuation of free variables, and thus are covered by the reduction to
the Boolean case.  We do consider \emph{unions of conjunctive queries
(UCQs)}, which are disjunctions of CQs. An interpretation $\Ii$
satisfies a UCQ $Q$ if it satisfies one of its disjuncts. It follows
immediately that UCQs are \emph{preserved under homomorphisms}; that
is, if $\Ii\models Q$ and there is a homomorphism from $\Ii$ to \Jmc,
then also $\Jmc\models Q$.

\medskip

A query $Q$ is \emph{entailed by a KB $\Kk$}, denoted as $\Kk\models
Q$, if every model of $\Kk$ satisfies $Q$. A model of $\Kk$ that does
not satisfy $Q$ is called a \emph{counter-model}. The \emph{query
entailment problem} asks whether a KB $\Kk$ entails a (U)CQ
$Q$. Moreover, this problem is equivalent to that of finding a
counter-model.  It is well known that the \emph{query answering
problem} can be reduced to query entailment.

In this paper, we address the problem of \emph{finite query
entailment}, which is a variant of query entailment where only finite
interpretations are considered: an interpretation \Imc
is \emph{finite} if $\Delta^\Imc$ is finite, and
a query $Q$ is \emph{finitely entailed by \Kmc}, denoted as $\Kmc
\models_{\sf{fin}} Q$, if every finite model of $\Kk$ satisfies $Q$.


\section{From tree-shaped to finite counter-models}
\label{sec:alcoi}



 



 

Let us fix an $\ALCOI$ knowledge base $\Kk$  and a union of conjunctive queries $Q$. Because we have nominals in our logic, we can assume without loss of generality that $\Kk$'s ABox does not contain role assertions. 

The construction of a finite counter-model begins from a tree-shaped counter-model. An interpretation $\Ii$ is \emph{tree-shaped} if the interpretation $\Ii\setminus \Nomi(\Kk)$ is a finite collection of trees of bounded degree, with elements of $\Ind(\Kk) \setminus \Nomi(\Kk)$ occurring only in the roots. It is well known that a tree-shaped counter-model can be obtained from an arbitrary counter-model $\Mm$ by the standard unravelling procedure. 
To turn a tree-shaped counter-model into a finite counter-model we use a variant of the \emph{blocking principle}: a systematic policy of reusing elements. For example, rather than adding a fresh $r$-successor of unary type $\tau$, one could add an $r$-edge to some previously added element of unary type $\tau$ (if there is one). This would give a finite model for $\Kk$, but not necessarily a counter-model for $Q$: a query asking for a cycle of length 42 might be unsatisfied in the original model, but the blocking principle introduces many new cycles, possibly one of length 42 among them. This is in fact the key difficulty to overcome: we need a blocking principle that does not introduce cycles shorter than the size of the query. 

The first step is to look at sufficiently large neighbourhoods, rather than just unary types.

\begin{definition} \label{def:neighbourhood}
For $d \in\Delta^\Ii \setminus \Nomi(\Kk)$, the $n$-neighbourhood $N_n^{\Ii}(d)$
is the subinterpretation of $\Ii$ induced by $\Nomi(\Kk)$ and all elements $e \in \Delta^\Ii \setminus \Nomi(\Kk)$ within distance $n$ from $d$ in $\Ii \setminus \Nomi(\Kk)$, enriched with a fresh concept interpreted as $\{d\}$. For $a \in \Nomi(\Kk)$, $N_n^{\Ii}(a)$ is the subinterpretation induced by $\Nomi(\Kk)$, enriched similarly. 
\end{definition}


Replacing unary types with large neighbourhoods is not enough, because nearby elements can have arbitrary large isomorphic  neighbourhoods: in the integers with the successor relation all $n$-neighbourhoods are isomorphic. The next step is to enrich the initial counter-model in such a way that overlapping neighbourhoods are not isomorphic, following an idea from \cite{GogaczM13}.

\begin{definition} \label{def:coloring} A \emph{colouring with $k$ colours} of an interpretation $\Ii$ is an extension $\Jj$ of $\Ii$ with $\Delta^\Jmc = \Delta^\Imc$, such that $\Jmc$ coincides with \Imc in every element in the signature of \Imc, and  
interprets fresh $k$ concept names $B_1, \dots, B_k$ such that $B_1^\Jj, \dots, B_k^\Jj$ is a partition of $\Delta^\Jmc$. We say that 
$d \in B_i^\Jj$ has colour $B_i$. 
A colouring \Jmc of $\Ii$ is \emph{$n$-proper} if for each $d \in \Delta^\Jmc$ 
all elements of $N_n^{\Jj}(d)$ have different colours. 
\end{definition}

Because $\Nomi(\Kk)$ is contained in each neighbourhood, in $n$-proper colourings each nominal has a unique colour.  

\begin{lemma} \label{lem:colouring}
If $\Ii\setminus \Nomi(\Kk)$ has bounded degree, then for all $n\geq 0$ there exists an $n$-proper colouring of $\Ii$ with finitely many colours. 
\end{lemma}

We write $\Ii_n$ for an arbitrarily chosen $n$-proper colouring of $\Ii$. Because the neighbourhoods have bounded size and we used only finitely many colours, there are only finitely many $n$-neighbourhoods in $\Ii_n$ up to isomorphism. The blocking principle described below relies on this. 

Let $\Ii$ be a tree-shaped counter-model for $Q$. We turn it into a finite counter-model for $Q$ as follows.  Because $\Ii\setminus \Nomi(\Kk)$ has bounded degree, we can consider an $n$-proper colouring $\Ii_n$ of $\Ii$. For each branch $\pi$ in $\Ii_n \setminus \Nomi(\Kk)$, let $d_\pi$ be the first node on $\pi$ such that some earlier node $e_\pi$ on $\pi$ satisfies $N^{\Ii_n}_n(d_\pi)\simeq N^{\Ii_n}_n(e_\pi)$.  The new interpretation $\Ff_n$ is obtained as follows. $\Ff_n\setminus \Nomi(\Kk)$ includes the branch $\pi$ up to the predecessor of node $d_\pi$ and the edge originally leading to $d_\pi$ is redirected to $e_\pi$. Because the degree in $\Ii_n\setminus \Nomi(\Kk)$ is bounded, the domain of $\Ff_n\setminus \Nomi(\Kk)$ is a finite subset of the domain of $\Ii_n\setminus \Nomi(\Kk)$. The whole interpretation $\Ff_n$ is obtained by including $\Nomi(\Kk)$ into the domain and copying from $\Ii_n$ all edges connecting elements of $\Nomi(\Kk)$ with each other and with the elements of $\Ff_n\setminus \Nomi(\Kk)$. 

Because we started from a model of $\Kk$, for all $n\geq 0$, \[\Ff_n \models \Kk\,.\] We claim that for sufficiently large $n$, $\Ff_n$ is a counter-model for $Q$. In order to prove this, we introduce yet another interpretation, containing $\Ii_n$ and $\Ff_n$ as subinterpretations.  

\begin{definition} \label{def:link}
Let $i\leq n$ and let $d$, $e$ be elements of $\Ii_n $. We say that $(d,e)$ is an $i$-link along role $r$ if either $d$ has an $r$-successor $e'$ in $\Ii_n$ such that $N^{\Ii_n}_i(e')\simeq N^{\Ii_n}_i(e)$,
or $e$ has an $r$-predecessor $d'$ in $\Ii_n$ such that $N^{\Ii_n}_i(d') \simeq N^{\Ii_n}_i(d)$.
\end{definition}

Notice that for $i<j$,  each $j$-link is also an $i$-link.  Note also that $(d,e)$ is an $i$-link along role $r$ if and only if $(e,d)$ is an $i$-link along $r^-$.

\begin{definition} \label{def:linked}
For $i\leq n$, let $\Ii_n^i$ be the interpretation obtained from $\Ii_n$ by including into the interpretation of each role $r$ all $i$-links along $r$; that is, for every role $r$ and  every $i$-link $(d,e)$ along $r$, $(d,e) \in r^{\Ii_n^i}$.
\end{definition}

Clearly, we have  \[ \Ii_n \subseteq \Ii_n^n \subseteq \Ii_n^{n-1} \subseteq \dots \subseteq  \Ii_n^1 \subseteq \Ii_n^0\,,\]
but the domains of all these interpretations coincide. We keep referring to the edges present in $\Ii_n^i$ but not in $\Ii_n$ as $i$-links, even though they are ordinary edges now. 

\begin{theorem} \label{thm:lifting}
Let $P$ be a CQ with at most $k$ binary atoms and let $n\geq k^2$. For each homomorphism $h : P \to \Ii_n^n$ there exists a homomorphism $h' : P \to \Ii_n$ such that \[N^{\Ii_n}_{n-k^2}(h(x)) \simeq N^{\Ii_n}_{n-k^2}(h'(x))\] for all $x\in \mathit{var}(P)$. 
\end{theorem}


\noindent Theorem~\ref{thm:lifting} holds for any interpretation $\Ii$ of any $\SOIF$ KB. 

Before proving Theorem~\ref{thm:lifting}, let us see that it implies that  $\Ff_{k^2} \notmodels Q$, where $k$ is a  common upper bound on the number of binary atoms in the CQs constituting $Q$.  Because $\Ff_{k^2}$ is obtained from  $\Ii_{k^2}$ by adding some $k^2$-links and restricting the domain, it follows that $\Ff_{k^2} \subseteq \Ii_{k^2}^{k^2}$.  Consequently, if there were a homomorphism $h: P \to \Ff_{k^2} \subseteq \Ii_{k^2}^{k^2}$ for some CQ $P$ constituting $Q$,  Theorem~\ref{thm:lifting} would yield a homomorphism $h' : P \to \Ii_{k^2}$, contradicting $\Ii \notmodels Q$.
Thus, we have proved finite controllability for $\ALCOI$.

\begin{corollary} \label{cor:ALCOI-FC}
 For each $\ALCOI$ KB  $\Kk$ and UCQ $Q$, 
\[\K \models Q \text{ iff } \K \fentails Q\,. \]
\end{corollary}

\begin{proof}[Proof of Theorem~\ref{thm:lifting}]  
Let $h(P)$ denote the subinterpretation of $\Ii_n^n$ obtained by restricting the domain to $h(\var(P))$, and only keeping in each role $r$ edges $(h(x), h(y))$ such that $r(x,y)$ is an atom from $P$. We say that $h$ \emph{uses} an $r$-edge of $\Ii_n^n$ if this $r$-edge is present in $h(P)$. 

Let $\ell$ be the number of links in $\Ii_n^n$ used by $P$. Then $\ell \leq k$, because $P$ contains at most $k$ binary atoms. The theorem follows by applying the following claim $\ell$ times: For each homomorphism $h:P \to \Ii_n^i$ with \mbox{$k\leq i \leq n$} that uses at least one link, there exists a homomorphism \mbox{$h':P \to \Ii_n^{i-k}$} that uses strictly fewer links and satisfies \[N^{\Ii_n}_{i-k}(h(x)) \simeq N^{\Ii_n}_{i-k}(h'(x))\] for all $x\in var(P)$. Let us prove the claim. 

Let $(d,e)$ be a link used by $h$: an $s$-edge in $h(P) \subseteq \Ii_n^i$ that is not an $s$-edge in $\Ii_n$. 
Then $(d,e)$ is an $i$-link in $\Ii_n$. By symmetry it suffices to consider the case when $d$ has an $s$-successor $e'$ in $\Ii_{n}$ such that  $N^{\Ii_n}_{i}(e)\simeq N^{\Ii_n}_{i}(e')$.  Let \[g: N^{\Ii_n}_{i}(e) \to N^{\Ii_n}_{i}(e')\] be the witnessing isomorphism. 
Because $g$ is identity over $\Nomi(\Kk) \subseteq \Ind(\Kk)$, we have $e\notin \Nomi(\Kk)$; indeed, otherwise $e'=g(e)=e$ and $(d,e)$ would be an $s$-edge in $\Ii_n$.
Let $E$ be the connected component of $e$ in \[h(P) \cap (\Ii_{n} \setminus \Nomi(\Kk))\,,\] where by $\Jj' \cap \Jj''$ we mean the interpretation $\Jj$ such that $\Delta^\Jj = \Delta^{\Jj'} \cap \Delta^{\Jj''}$, $A^\Jj = A^{\Jj'} \cap A^{\Jj''}$ for all concept names $A$, and $r^\Jj = r^{\Jj'} \cap r^{\Jj''}$ for all role names $r$. Because $h(P)$ has at most $k$ edges and $(d,e)$ is an $s$-edge in $h(P)$ but not in $E$, there are at most $k-1$ edges in  $E$. We shall bring $E$ close to $d$ in $\Ii_{n}$ by pulling it back by the $i$-link $(d,e)$.

As $E$ is a connected subinterpretation of $\Ii_{n} \setminus \Nomi(\Kk)$ and has at most $k-1$ edges, each element of $E$ lies within distance $k-1$ from $e$. In particular, $E \subseteq N^{\Ii_n}_{i}(e)$. Hence, $E$ is contained in the domain of $g$ and we can define \[h':P \rightarrow \Ii_{n}^{i-k}\] as follows. For each $x\in \var(P)$, let $h'(x)=g(h(x))$ if $h(x)\in E$, and $h'(x)=h(x)$ otherwise. 
The additional claim of the theorem follows immediately because $g$ preserves $(i-k)$-neighbourhoods of elements within distance $k$ from $e$. We only need to verify that $h'$ is indeed a homomorphism and that it uses fewer links than $h$.

Let $r(x,y)$ be an atom of the query $P$. There are three cases to consider.
First, suppose that $h(x),h(y) \notin E$. Then \[(h'(x),h'(y))=(h(x),h(y))\,.\] We have that  $(h(x),h(y))$ is an $r$-edge in $\Ii_{n}^{i-k}$ because $h$ is a homomorphism into $\Ii_{n}^{i} \subseteq \Ii_{n}^{i-k}$. Obviously, $h'$ uses no new links for such atoms.

Next, suppose that $h(x),h(y) \in E$. Then \[(h'(x),h'(y))=(g(h(x),g(h(y)))\,.\] Moreover, $(h(x),h(y))$ is an $r$-edge in $\Ii_{n}^{i}$ because $h$ is a homomorphism. Suppose it is a link along $r$. Then, $h(x)$ has an $r$-successor in $\Ii_n$ with the same colour as $h(y)$, or $h(y)$ has an $r$-predecessor in $\Ii_n$ with the same colour as $h(x)$.  
Because both $h(x)$ and $h(y)$ lie within distance $k-1$ from $e$, this successor or predecessor belongs to $N^{\Ii_n}_{i}(e)$, along with $h(x)$ and $h(y)$. But this is impossible because all elements of  $N^{\Ii_n}_{i}(e)$ have different colours. Hence, $(h(x),h(y))$ is an $r$-edge in $N^{\Ii_n}_{i}(e)$ and $(g(h(x)),g(h(y)))$ is an $r$-edge in $N^{\Ii_n}_{i}(e')$. That is, $(g(h(x)),g(h(y)))$ is an $r$-edge in $\Ii_{n}^{i-k}$, and is not a link along $r$.  

Finally, suppose that $h(x)\notin E$ and $h(y)\in E$ (the symmetric case is analogous).  Because $h$ is a homomorphism,  $(h(x),h(y))$ is an $r$-edge in $\Ii_{n}^{i}$. Now there are two subcases. Assume first that $(h(x),h(y))$ is also an $r$-edge in $\Ii_n$. By the definition of $E$ it is not an $r$-edge in $\Ii_n\setminus \Nomi(\Kk)$, so it must be an $r$-edge between a nominal and an element of $E$. As such, it is also an $r$-edge in $N^{\Ii_n}_{i} (e)$. Consequently, \[(h'(x), h'(y)) = (h(x), g(h(y))) = (g(h(x)),g(h(y)))\] is an $r$-edge in $N^{\Ii_n}_{i} (e')$ and we conclude like previously.

Assume now that $(h(x),h(y))$ is an $i$-link along $r$. We need to check that $(h(x),g(h(y)))$ is an $r$-edge in $\Ii_{n}^{i-k}$.  Since $h(y)$ and $g(h(y))$ are in distance at most $k-1$ from $e$ and $e'$, respectively, and $N^{\Ii_n}_{i}(e)\simeq N^{\Ii_n}_{i}(e')$, it follows that \[N^{\Ii_n}_{i-k}(h(y))\simeq N^{\Ii_n}_{i-k}(g(h(y)))\,.\] 
Because $(h(x), h(y))$ is an $i$-link, it is also an $(i-k)$-link. If $h(x)$ has an $r$-successor $f$ in $\Ii_{n}$ such that 
\[N^{\Ii_n}_{i-k}(f) \simeq N^{\Ii_n}_{i-k}(h(y)) \simeq N^{\Ii_n}_{i-k}(g(h(y)))\,,\]
then $(h(x),g(h(y)))$ is an $(i-k)$-link along $r$, unless the successor $f$ is $g(h(y))$ itself; in either case $(h(x), g(h(y)))$ is an $r$-edge in $\Ii_{n}^{i-k}$.
The remaining possibility is that $h(y)$ has an $r$-predecessor $f$ in $\Ii_{n}$ such that 
\[N^{\Ii_n}_{i-k}(f) \simeq N^{\Ii_n}_{i-k}(h(x)) \,.\]
Because $h(y)$ lies within distance $k-1$ from $e$, 
\[N^{\Ii_n}_{i-k}(f) \subseteq N^{\Ii_n}_{i}(e) \,.\]
Hence, $g(f)$ is an $r$-predecessor of $g(h(y))$ such that 
\[N^{\Ii_n}_{i-k}(g(f)) \simeq N^{\Ii_n}_{i-k}(h(x)) \,.\] 
Consequently, $(h(x),g(h(y)))$ is an $(i-k)$-link along $r$, unless $g(f)$ is $h(x)$ itself; in either case $(h(x), g(h(y)))$ is an $r$-edge in $\Ii_{n}^{i-k}$.

Thus $h'$ is a homomorphism and uses links only for the atoms of $P$ for which $h$ uses links. To see that $h'$ uses strictly fewer links than $h$, recall that instead of the $i$-link $(d,e)$ along $s$, it uses the $s$-edge $(d, e')$, which is not a link.
\end{proof}


\section{$\SOI$ and $\SOF$ }
\label{sec:soi}

The goal of this section is to prove the following theorem. 
\begin{theorem}
The finite query entailment problem for both $\SOI$ and $\SOF$ is \twoexp-complete.
\end{theorem}

The lower bounds follow immediately from the results on unrestricted query entailment for $\ALCO$ \cite{NgoOS16} and $\ALCI$ \cite{Lutz08}, and Corollary~\ref{cor:ALCOI-FC}; the challenge is to prove the upper bounds.  We develop our argument with $\SOI$ in mind, but it adapts easily to $\SOF$ (see appendix). 

Let us fix a $\SOI$ knowledge base $\Kk$ and a union of conjunctive queries $Q$.  Like for $\ALCOI$, we can assume that $\Kk$'s ABox contains no role assertions. 

Because $\Kk$ is normalised, complete information about restrictions on the types of neighbours of a node is encoded in its unary type. Now, we would like the unary type to determine also the neighbouring nominals. This can be assumed without loss of generality,  because one can always extend $\Kk$ by adding for each $a\in\Nomi(\Kk)$ and $r \in\Rol(\Kk)$ fresh concept names $A_{r,a}$, $A_{r^-,a}$ axiomatised with $A_{r,a} \equiv \exists r. \{a\}$, $\{a\} \equiv \forall r. A_{r^-,a}$, and normalise the resulting KB.

Let $\Ii^*$ be the interpretation obtained from interpretation $\Ii$ by closing transitively the interpretation of each transitive role. Note that each existential restriction satisfied in $\Ii$ is also satisfied in $\Ii^*$. The same holds for quantifier-free CI, and for universal restrictions involving non-transitive roles. For universal restrictions involving transitive roles, we ensure this property by adding a fresh concept name $B'$ for each $B\in \CN(\Kk)$ and CIs $A \sqsubseteq \forall r. B'$,  $B' \sqsubseteq  \forall r. B'$, $B' \sqsubseteq  B$ for each CI of the form $ A \sqsubseteq \forall r. B$ with $r$ transitive. 


The last assumption we would like to make about $\Kk$ is that the unary type of each element of $\Nomi(\Kk)$ is fully specified in the ABox; that is, for all $a \in\Nomi(\Kk)$ and  $A\in\CN(\Kk)$, the ABox contains either $A(a)$ or $\bar A(a)$. This can be done without loss of generality, because $\Kk \fentails Q$ iff $\Kk' \fentails Q$ for each $\Kk'$ that can be obtained from $\Kk$ by completing assertions about nominals. This adds the factor $2^{|\Nomi(\Kk)|\cdot|\CN(\Kk)|}$ to the running time of the decision procedure, but the overall complexity bound is not affected, because it is exponential in the size of $\Kk$ anyway.

 Building on the results of the previous section, we show that the existence of a finite counter-model for $Q$ is equivalent to the existence of a possibly infinite counter-model of a special form, which generalises tree-shaped models. 
The special form is based on the notion of clique-forests.
\begin{definition}
A clique-forest for an interpretation $\Ii$ of $\Kk$ is a forest (a sequence of trees) whose each node $v$ is labelled with a subinterpretation $\Ii_v$ of $\Ii\setminus \Nomi(\Kk)$ such that
\begin{itemize}
\item the sets $\Delta^{\Ii_v}$ are a partition of $\Delta^{\Ii \setminus \Nomi(\Kk)}$;
\item each $\Ii_v$ is either a single element with all roles empty (element node) or a clique over some transitive role with all other roles empty and no repetitions of unary types (clique node);
\item apart from edges within cliques,  in $\Ii \setminus \Nomi(\Kk)$ there is exactly one edge between $\Delta^{\Ii_u}$ and $\Delta^{\Ii_v}$ for every two adjacent nodes $u$ and $v$: assuming $u$ is the parent of $v$, it is an $r$-edge from an element of $\Delta^{\Ii_u}$ to an element of $\Delta^{\Ii_v}$ for some $r\in\Rol(\Kk)$.

\end{itemize}
\end{definition}

\begin{definition}
An interpretation $\Ii$ of $\Kk$ is a $\SOI$-forest if it admits a clique-forest that consists of at most $|\Kk|^2$ trees of branching at most $|\Kk|^2$, such that each element of $\Ind(\Kk)\setminus \Nomi(\Kk)$ occurs in some root.
\end{definition}



Let $\Kk^*$ denote the KB obtained from $\Kk$ by dropping transitivity declarations.

\begin{definition}
A \emph{counter-example} for $Q$ is a $\SOI$-forest $\Ii$ such that $\Ii\models \Kk^*$ and $\Ii^*\notmodels Q$.
\end{definition}

If $\Ii$ is a counter-example for $Q$, thanks to the initial preprocessing, $\Ii^*$ is a counter-model for $Q$. One could also show that if there is a counter-model for $Q$, then there is a counter-example for $Q$. But we are interested in \emph{finite} counter-models and for that we need an additional condition.  Recall that a path is simple if it does not revisit elements. 

\begin{definition}
An interpretation $\Ii$ is \emph{safe} if it does not contain an infinite simple $r$-path for any transitive role $r$.
\end{definition}

The whole argument now splits into two parts: equivalence of the existence of a \emph{finite} counter-model  and a \emph{safe} counter-example, and effective regularity of the set of clique-forests of safe counter-examples. Together they show that finite query entailment can be solved by testing emptiness of an appropriate doubly-exponential automaton (with B\"uchi acceptance condition), which can be done in polynomial time. We begin from the second part, as it is needed to prove the first one.

\begin{theorem}\label{thm:SOI-regular}
Given a union $Q$ of CQs, each of size at most $m$, one can compute (in time polynomial in the size of the output) an automaton of size $2^{|Q| \cdot |\Kk|^{\Oo(m)}}$ that recognises clique-forests of safe counter-examples for $Q$.  
\end{theorem}

The proof of Theorem~\ref{thm:SOI-regular} is a routine automata construction (detailed in the appendix). Let us focus on the first part of the argument.

\begin{theorem}\label{thm:SOI-characterization}
$Q$ has a finite counter-model iff $Q$ has a safe counter-example.
\end{theorem}

Suppose first that there exists a finite counter-model $\Mm$ for $Q$. We build a $\SOI$ forest $\Ii$ out of it using a version of the standard unravelling. We begin by taking copies of all elements of $\Ind(\Kk)$ with unary types copied accordingly. Then, recursively, for each added element $d'$ and each CI $A \sqsubseteq \exists r. B$ that is not yet satisfied for $d'$ in $\Ii$ proceed as follows. The element $d'$ is a copy of some $d$ from $\Mm$ of the same unary type. Therefore there exists an element $e$ in $\Mm$ witnessing the CI. If $e\in\Nomi(\Kk)$, then it is already included in $\Ii$, and we just add an $r$ edge from $d'$ to $e$. Assume $e\notin\Nomi(\Kk)$. If $r$ is not a transitive role, we just add a copy of $e$ as an $r$-successor of $d'$. Assume that $r$ is a transitive role. Let $X$ be the strongly connected component of $r$ that contains $e$ and let $X_0$ be a minimal set that contains at least one element from each nonempty $C^\Mm \cap \big (X \setminus\Nomi(\Kk)\big)$, where $C$ ranges over $\CN(\Kk)$. By minimality, $|X_0|\leq |\Kk|$. We add to $\Ii$ an $r$-clique over a copy of $X_0$, with an $r$ edge from $d'$ to the  copy of some element $f \in B^\Mm \cap X_0$; $f$ exists because $e\in B^\Mm \cap \big (X \setminus\Nomi(\Kk)\big)$. 
Note that no other edges among newly added elements are present: existential restrictions for these nodes will be witnessed in the following steps of the construction. Let $\Ii$ be the interpretation obtained in the limit. 
By construction, $\Ii$ admits a clique-forest. For each element at most one successor per CI is added.  Because each clique node contains up to $|\Kk|$ elements, the branching of the clique-forest is bounded by $|\Kk|^2$. The same bound holds for the number of trees in the clique-forest: we begin from $|\Ind(\Kk)|$ nodes, but then the ones corresponding to elements of $\Nomi(\Kk)$ are removed and their children become roots. Hence, $\Ii$ is a $\SOI$ forest. Because we do not unravel cliques in transitive roles, it is safe.

\begin{lemma} \label{lem:SOI-unrev}
 $\Ii$ is a safe counter-example for $Q$. 
\end{lemma}

Assume now that there exists a safe counter-example $\Ii$ for $Q$. By Theorem~\ref{thm:SOI-regular}, the set of 
clique-forests of safe counter-examples for $Q$
can be recognised by an automaton. It is well known that the automaton then accepts a regular forest, which has only finitely many non-isomorphic subtrees. Hence, without loss of generality we can assume that the clique-forest of $\Ii$ has $p$ non-isomorphic subtrees for some $p$. Using the methodology from the previous section we shall turn $\Ii$ into a finite counter-model for $Q$. The main obstacle is that $Q$ uses transitive roles, which are not fully represented in $\Ii$. Our solution is to replace $Q$ with a different query that can be evaluated directly over $\Ii$. This is done by exploiting a bound on the length of simple $r$-paths for transitive roles $r$, guaranteed by the regularity of the clique-forest of $\Ii$.



\begin{definition}
An interpretation is \emph{$\ell$-bounded} if for each transitive role $r$, each simple $r$-path has length at most $\ell$. 
\end{definition}

\begin{lemma} \label{lem:paths}
$\Ii \setminus \Nomi(\Kk)$ is $\ell$-bounded for $\ell = 2 p \cdot |\Kk|$.
\end{lemma}
\begin{proof}
Let $r$ be a transitive role in $\Kk$. Each $r$-path going down the
clique-forest of $\Ii$ contains at most $p$ nodes. Indeed, if there
were a longer $r$-path, then a subtree would occur twice on that path,
which immediately leads to an infinite simple $r$-path in
$\Ii\setminus \Nomi(\Kk)$, contradicting the safety of $\Ii$. Each
simple path in the clique-forest can be split into an $r$-path going
up and an $r$-path going down. Each of them has at most $p$ nodes.
Because each node contains at most $|\Kk|$ elements, it follows that
each simple $r$-path in $\Ii \setminus \Nomi(\Kk)$ has length at most
$ 2 p\cdot |\Kk|$. 
\end{proof}

\begin{lemma} \label{lem:paths-nominals}
For each $\Jj$, if $\Jj \setminus \Nomi(\Kk)$ is $\ell$-bounded,  then $\Jj$ is $\ell^*$-bounded for $\ell^*=(\ell+2)\cdot (|\Nomi(\Kk)| + 1)$.
\end{lemma}

Let $Q^*$ be obtained from $Q$ by replacing each transitive atom $s(x,y)$ by the disjunction \[\bigvee_{i\leq \ell^*} s^i(x,y)\,,\] where $s^i(x,y)$ is the conjunctive query expressing the $i$-fold composition of $s$. Assuming that each disjunct of $Q$ contains at most $k$ binary atoms, $Q^*$ can be rewritten as a union of conjunctive queries, each using at most $k\cdot \ell^*$ binary atoms.

\begin{lemma} \label{lem:star}
For all $\ell^*$-bounded $\Jj$, $\Jj^*\models Q$ iff $\Jj\models Q^*$. 
\end{lemma}

By Lemmas \ref{lem:paths}--\ref{lem:star}, we conclude that $\Ii \notmodels Q^*$. Now we can use the blocking principle. Because clique nodes have at most $|\Kk|$ elements and each node has at most $|\Kk|^2$ children, $\Ii\setminus \Nomi(\Kk)$ has bounded degree and we can consider the $n$-properly coloured $\Ii_n$, for any $n$. On each branch $\pi$ in $\Ii_n \setminus \Nomi(\Kk)$, let $D_\pi$ be the first node for which some earlier node $E_\pi$ satisfies $N^{\Ii_n}_n(d_\pi)\simeq N^{\Ii_n}_n(e_\pi)$, where $d_\pi \in D_\pi$ and $e_\pi\in E_\pi$ are the endpoints of the edges connecting $D_\pi$ and $E_\pi$ to their parent nodes. The new interpretation $\Ff_n$ is obtained as usual: we include the branch $\pi$ up to the predecessor of node $D_\pi$ and the edge originally leading to $d_\pi$ is redirected to $e_\pi$; edges connecting the elements of $\Nomi(\Kk)$ with each other and with the elements of the included parts of the branches are copied from $\Ii_n$. 

Because we started from $\Ii\models\Kk^*$, it is routine to check that $\Ff_n \models \Kk^*$ for all $n$. By the initial preprocessing, $(\Ff_n)^* \models \Kk$.
%
%
Let us fix \[n=\max((k\cdot \ell^*)^2, (\ell+1)^2+\ell)\,.\] By Theorem~\ref{thm:lifting}, $\Ff_n \notmodels Q^*$. We conclude $(\Ff_n)^* \notmodels Q$ using Lemmas~\ref{lem:paths-nominals}--\ref{lem:star} and Theorem~\ref{thm:folding-bounded} below.

\begin{definition} 
A link $(d,e)$ in $\Ii$ along $r$ is \emph{external} if either no $r$-path from the witnessing $e'$ to $d$ is disjoint from $\Nomi(\Kk)$ or dually no $r$-path from $e$ to the witnessing $d'$ is disjoint from $\Nomi(\Kk)$.
\end{definition}
By construction, all links in $\Ii_n$ along transitive roles included into $\Ff_n$ are external.


\begin{theorem} \label{thm:folding-bounded}
Assume that  $\Ii\setminus \Nomi(\Kk)$ has bounded degree and is $\ell$-bounded. Let $n>(\ell+1)^2 + \ell$ and let $\Jj$ be a subinterpretation of $\Ii_n^n$ in which all links along transitive roles are external. Then, $\Jj\setminus \Nomi(\Kk)$ is also $\ell$-bounded.
\end{theorem}

\begin{proof}
Suppose there is a simple $s$-path $\pi$ in $\Jj\setminus \Nomi(\Kk)$ of length $\ell+1$, for some transitive role $s$. We can view $\pi$ as a conjunctive query with $\ell+1$ $s$-atoms. By applying Theorem~\ref{thm:lifting} to $\pi$ we lift the inclusion homomorphism $\pi \subseteq \Jj \subseteq \Ii_n^n$ to a homomorphism $ h :  \pi \to \Ii_n\,,$
that preserves $\ell$-neighbourhoods. 
Because $\pi$ is disjoint from $\Nomi(\Kk)$, so is its image. 
Hence, we can view $h$ as a homomorphism  
\[ h :  \pi \to \Ii_n \setminus \Nomi(\Kk) \,.\] 
Because $\Ii_n\setminus \Nomi(\Kk)$ is $\ell$-bounded, it suffices to show that $h$ is injective to obtain a contradiction. 

Observe first that $h$ is injective over segments of $\pi$ that do not contain links. Indeed, because $\Ii_n$ is $n$-properly coloured and $n\geq |\pi|$, in each such segment all elements have different colours. Hence, it suffices to show that the images of the segments are disjoint. Suppose the images of some two different segments overlap on an element from a strongly connected component $X$ of $s$ in $\Ii_n\setminus \Nomi(\Kk)$. Hence, all segments between these two are entirely mapped to $X$. In particular, there exists an $n$-link $(d,e)$ along $s$ such that $h(d) \in X$ and $h(e) \in X$. We claim this is impossible. 

By symmetry we can assume that $d$ has an $s$-successor $e'$ such that no $s$-path from $e'$ to $d$ is disjoint from $\Nomi(\Kk)$ and $N^{\Ii_n}_n(e') \simeq N^{\Ii_n}_n(e)$. In particular, $e'$ and $e$ have the same colour. Because $n>1$, we have $e' \in N^{\Ii_n}_n(d)$. We obtain a contradiction by finding another element in $N^{\Ii_n}_n(d)$ of the same colour as $e$. 

Let $D$ be the strongly connected component of $s$ in $\Ii_n\setminus\Nomi(\Kk)$ that contains $d$. Because $\Ii_n\setminus\Nomi(\Kk)$ is $\ell$-bounded, all elements of $D$ are within distance $\ell<n$ from $d$. Consequently, $D$ is isomorphic to $X$,  because $h$ preserves $\ell$-neighbourhoods. Hence, there exists an element $e''\in D \subseteq N^{\Ii_n}_n(d)$ of the same colour as $e$. Because $e'\notin D$, we have $e'\neq e''$, as required for the contradiction. 
\end{proof}









\section{$\SIF$}
\label{sec:sif} 

For $\ALCIF$, a tight upper bound on the complexity of finite query
entailment can be obtained by revisiting some known and implicitly
proven results on the guarded fragment with two variables and counting
\cite{Pratt-Hartmann_datacomplexity,Pratt-Hartmann0_gc2_sat}. We
consider a slightly more general problem of \emph{finite entailment
  modulo types}, which will be useful later. For a KB $\Kk$, a query
$Q$, and a set of unary types $T \subseteq \Tp(\Kk)$ we write $\Kk
\fentails^T Q$ if for each interpretation $\Ii$ that only realises
types from $T$, if $\Ii\models \Kk$ then $\Ii\models Q$. This problem
reduces to finite query entailment by including into $Q$ one CQ
for each type not listed in $T$, but this makes $Q$ exponential in the
size of  $\CN(\Kk)$ and leads to a worse complexity upper bound.  

\begin{theorem}\label{thm:ALCIF}
Given an $\ALCIF$ KB $\Kk$, a union $Q$ of  CQs, each of size at most
$m$, and a set $T\subseteq \Tp(\Kk)$, one can decide whether $\Kk
\fentails^T Q$ in time $2^{\Oo(|\K|+ |Q| \cdot m^m)}$. 
\end{theorem}

\begin{corollary} \label{cor:ALCIF}
The finite query entailment problem for $\ALCIF$ is \twoexp-complete.
\end{corollary}

\noindent Relying on Theorem~\ref{thm:ALCIF} and our previous results
for $\SOI$, we extend the upper bound of Corollary~\ref{cor:ALCIF} to
$\SIF$.

Let us fix a UCQ $Q$ and a $\SIF$ KB $\Kk$.
%
%
%
%
%
We work again with counter-models of a special shape, this time based
on tree partitions. We assume a proviso that the ABox of $\Kk$
does not contain transitive and non-transitive roles simultaneously;
we lift it by the end of the section.

\begin{definition}
A \emph{tree partition} of an interpretation $\Ii$ is a tree
$T$ whose each node $v$ is labelled with a finite subinterpretation $\Ii_v$ of
$\Ii$, called a \emph{bag}, such that $\bigcup_{v\in T} \Ii_v = \Ii$ and for
each element some bag containing it is the parent of all other bags
containing it. The maximal bag size is called the \emph{width} of $T$.
\end{definition}


\begin{definition}
An interpretation $\Ii$ is a \emph{$\SIF$-tree} if it admits a
tree partition such that 
\begin{itemize}
\item the root bag contains $\Ind(\Kk)$,
\item each bag contains edges in transitive roles only ($\tra$
  bag) or in non-transitive roles only 
($\ntra$ bag),
\item each element is in exactly two bags, one $\tra$ and one $\ntra$,
\item each two adjacent bags share exactly one element.
\end{itemize}
\end{definition}

\begin{lemma} \label{lem:SIF-trees}
There exists a finite counter-model for $Q$ iff there
exists a $\SIF$-tree counter-model for $Q$ of finite width. 
\end{lemma}

\begin{proof}
Let $\Ff$ be a finite counter-model for $Q$. We turn it into a
$\SIF$-tree counter-model $\Ii$ using a very simple unravelling
procedure. For each $\mu\in\{\tra, \ntra\}$, let $\Ff_\mu$ be the
interpretation obtained from $\Ff$ by restricting the set of roles to
$\mu$ roles. By the proviso, the ABox of $\Kk$ contains only $\mu_0$
roles for some $\mu_0\in \{\tra, \ntra\}$. We construct the
$\SIF$-tree top down. In the root we put $\Ff_{\mu_0}$ itself. Then,
iteratively, for each element $d$ that belongs only to a $\mu$ bag we
add a child bag obtained by taking an isomorphic copy of $\Ff_\nu$ for
$\nu\neq \mu$, in which all elements except $d$ are replaced with
their fresh copies; in particular, each individual different from $d$ is
replaced with an ordinary anonymous element of the same unary
type. It is routine to verify that the resulting
interpretation $\Ii$ is a model of $\Kk$.
The natural homomorphism from
$\Ii$ to $\Ff$ ensures that $\Ii\notmodels Q$. The width of $\Ii$ is
$|\Ff|$.

Let us now take a $\SIF$-tree $\Ii$ of width $\ell$ that is a
counter-model for $Q$. We use the methodology developed for $\SOI$ to
turn $\Ii$ into a finite counter-model. Because $|\Ii_v| \leq \ell$,
$\Ii$ has degree at most $2\cdot \ell\cdot |\Kk|$. Because each
$r$-path for any transitive role $r$ is contained within a single
$\tra$ bag, it follows that $\Ii$ is $(\ell-1)$-bounded.


For the purpose of the coloured blocking principle, we need to ensure
that each infinite branch of the tree partition of our interpretation
contains infinitely many $\tra$ bags that consist of a single edge
(pointing up or down the tree). We achieve this by performing an
additional unravelling of $\Ii$. We start with a copy of the root bag
in the tree partition of $\Ii$, where elements of $\Ind(\Kk)$ are
preserved and other elements are replaced with their fresh copies. Let
$d'$ be an element in the interpretation under construction that so
far belongs to only one bag $X'$. By construction, $d'$
is a copy of some element $d$ of $\Ii$. If $X'$ is a $\tra$ bag, add a
copy of the $\ntra$ bag that contains $d$, with $d$ replaced with $d'$
and other elements replaced with their fresh copies. Assume that $X'$
is an $\ntra$ bag. For each $\tra$ role $r$ and each $r$-successor $e$
of $d$, add three new bags. First, add a bag consisting of $d'$, a
fresh copy $e'$ of $e$, and an $r$-edge from $d'$ to $e'$. Then, for
each $\mu\in\{\tra, \ntra\}$, add a copy of the $\mu$-bag containing
$e$, with $e$ replaced with $e'$ and all other elements replaced with
their fresh copies (different for each $\mu$).

Let $\Jj$ be the interpretation obtained in the limit. Because in the
tree partition of $\Ii$ $\tra$ bag and $\ntra$ bags alternate, in the
tree partition of $\Jj$ $\ntra$ bags have only new single-edge
$\tra$ bag children, new single-edge $\tra$ bags have one $\ntra$ bag
child and one $\tra$ bag child, and copies of original $\tra$ bags have only
$\ntra$ bag children.  Consequently, on each infinite branch, there
are infinitely many single-edge $\tra$ bags. 

Interpretations of transitive roles in $\Jj$ need not be transitive
relations, but it is straightforward to check that $\Jj$ is a model of
$\Kk^*$; in particular, functionality declarations were not affected
because the new single-edge bags involve only $\tra$ roles
(non-functional). Moreover, $\Jj^*\notmodels Q$ because $\Jj$ maps
homomorphically to $\Ii$ and, consequently, so does $\Jj^*$. The
degree in $\Jj$ is bounded by $2\cdot\ell\cdot|\Kk| +1$, because each
element belongs to one $\tra$ bag and one $\ntra$ bag of size at most
$\ell$, and possibly one single-edge bag. Finally, $\Jj$ is
$2\ell$-bounded because in the worst case a simple $r$-path for any
transitive role $r$ goes first through a bag with at most  $\ell$ elements,
then two single-edge bags, and then another bag with at most $\ell$
elements.

We can now apply the coloured blocking principle. Suppose each disjunct of
$Q$ uses at most $k$ binary atoms.  Let $\ell^* = 2\ell$ and let $Q^*$
be obtained from $Q$ by replacing each transitive role atom $S$ by the
disjunction \[\bigvee_{i\leq \ell^*} S^i(x,y)\,,\] and rewriting the
resulting query as a UCQ. Each CQ in $Q^*$ has at most $k\cdot \ell^*$
binary atoms. Because $\Jj$ has bounded degree, we can consider its
$n$-proper colouring $\Jj_n$ for any $n$. On each branch $\pi$ of the
tree partition of $\Jj_n$, let $D_\pi$ be the first single-edge
$\tra$ bag for which some earlier single-edge $\tra$ bag $E_\pi$
satisfies $N^{\Jj_n}_n(d_\pi) \simeq N^{\Jj_n}_n(e_\pi)$, where $d_\pi\in D_\pi$ and
$e_\pi\in E_\pi$ are the elements that $D_\pi$ and $E_\pi$ share with
their respective parents. The new structure $\Ff_n$ is obtained like
before: we include the branch $\pi$ up to the predecessor of node
$D_\pi$ and the edge in $D_\pi$ is redirected to the successor of
$e_\pi$ in $E_\pi$. Because $\Jj$ is a model of $\Kk^*$ and we only
redirected edges in non-functional roles, it follows that $\Ff_n$ is a
model of $\Kk^*$. Consequently, $\Ff_n^*\models \Kk$. Let us now
fix \[n=\max((k\cdot \ell^*)^2, (\ell^*+1)^2+\ell^*)\,.\] By
Theorem~\ref{thm:lifting}, we get $\Ff_n \notmodels \Qq^*$. Because
$\Jj$ is $\ell^*$-bounded and we clearly used only external links in
the construction of $\Ff_n$, by Lemma~\ref{lem:star} and
Theorem~\ref{thm:folding-bounded} we obtain $\Ff_n^* \notmodels \Qq$.
\end{proof}

Thus, it suffices to consider counter-models that are $\SIF$-trees of
finite width, but there is a priori no bound on the width, which
hinders direct application of automata. Instead, we show that one can
test existence of $\SIF$-tree counter-models without manipulating
$\SIF$-trees directly.

Our first step is to adjust the structure of $Q$'s disjuncts to the
structure of $\SIF$-trees. To keep this as simple as possible, we make
a second proviso that each CQ constituting $Q$ is connected. 
We eliminate it towards the end of the section. 
Let $P$ be one of the CQs
constituting $Q$. It is convenient to think $P$ as an interpretation
with the domain $\var(P)$ and interpretations of concepts and roles
given by the atoms of $P$. Whenever $P$ is mapped homomorphically into
a $\SIF$-tree $\Ii$, the image of $P$ is a $\SIF$-tree as
well. Indeed, because $P$ is connected, a witnessing tree partition of
the image of $P$ is naturally induced by the tree partition of
$\Ii$. Hence, if $Q$ is a union of $n$ CQs of size at most $m$, over
$\SIF$-trees $Q$ is equivalent to
\begin{equation} \label{eq:queries}
Q_1 \cup Q_2 \cup \dots \cup Q_p\,, \tag{*}
\end{equation} 
where the queries $Q_i$ are
all non-isomorphic $\SIF$-trees obtained as homomorphic images of the CQs
of $Q$, each using a fresh set of variables, and $p \leq n
\cdot m^m$. 

\begin{figure}
\centering
\includegraphics[height=4cm]{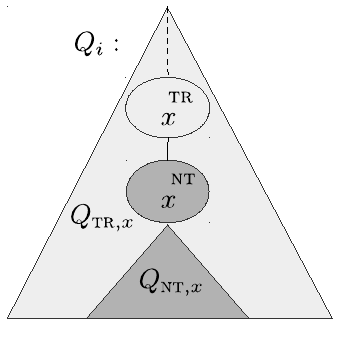} 
\caption{Queries $Q_{\tra,x}$ and $Q_{\ntra,x}$ for $x\in\var(Q_i)$.}
\label{fig:subqueries}
\end{figure}

For all $\mu \in \{\tra,\ntra\}$ and $x \in \bigcup_i \var(Q_i)$, let $Q_{\mu,
  x}$ be the query obtained by taking all bags that 
are reachable from the $\mu$ bag containing $x$ without visiting
the other bag containing $x$, as illustrated in Figure~\ref{fig:subqueries}. For all $x\in\var(Q_i)$ it holds that  $Q_i = Q_{\tra,x} \land
Q_{\ntra,x}$.

Let $\Kk_Q$ be obtained from $\Kk$ by extending the TBox as
follows: for each $\mu \in \{\tra, \ntra\}$ and $x \in \bigcup_i
\var(Q_i)$, we add a fresh concept name $A_{\mu,x}$
and the complementary concept name $\bar A_{\mu,x}$, together with the
usual axiomatisation. The interpretation of $A_{\mu,x}$ is intended to
collect elements $d$ such that $Q_{\mu, x}$ can be matched with $x$
mapped to $d$.

A \emph{specialisation} $\widetilde Z$ of a bag $Z$ of query $Q_i$ is
obtained by including for each $x \in \var(Z)$ and each $\mu \in
\{\tra, \ntra\}$ either the atom $A_{\mu,x}(x)$ or the atom $\bar
A_{\mu,x}(x)$, where $\bar A_{\mu,x}$ is the concept name
complementary to $A_{\mu,x}$.  A specialisation $\widetilde Z$ of a
$\mu$-bag $Z$ of $Q_i$ is \emph{consistent} if for all $x$ it holds that:
$\widetilde Z$ contains $A_{\mu,x}(x)$ iff 
for all $y\in \var(\widetilde Z) \setminus \{x\}$, $\widetilde Z$ contains
  $A_{\nu,y}(y)$  with $\nu\neq \mu$.
%
%
%
An interpretation $\Ii$  (with the extended set of
concept names) is \emph{consistent} if it does not match inconsistent
specialisations of bags of queries $Q_1, Q_2, \dots, Q_p$.

For a $\SIF$ KB $\Ll$ and $\mu\in\{\tra, \ntra\}$ we write
$\Ll \!\upharpoonright \! \mu$ for the KB obtained by dropping all ABox
assertions, CIs, and declarations that involve $\nu$-roles for $\nu\neq
\mu$.

\begin{definition} \label{def:counter-witness}
$T \subseteq \Tp(\Kk_Q)$ is a \emph{counter-witness} for $Q$ if
\begin{itemize}
\item for all  $x \! \in \!\bigcup_i \var (Q_i)$, each $\tau \!\in\! T$
  contains $\bar A_{\tra,x}$ or $\bar A_{\ntra,x}$;
\item assuming $\Kk$ uses only $\mu_0$-roles in the ABox, there exists
  a consistent finite model of $\Kk_Q \!\upharpoonright\! \mu_0$ that
  realises only types from $T$; and
\item for all $\tau \in T$ and $\mu\in\{\tra, \ntra\}$  there
  exists a consistent finite model of the TBox of $\Kk_Q
  \!\upharpoonright\! \mu$ that realises type $\tau$ and realises only
  types from $T$. 
\end{itemize} 
\end{definition}

\begin{lemma}\label{lem:sif-characterization}
$Q$ admits a $\SIF$ tree counter-model of finite width iff
there exists a counter-witness for $Q$. 
\end{lemma}

\begin{proof}
Let $\Ii$ be a $\SIF$-tree counter-model for $Q$; we do not need to
assume that $\Ii$ has finite width. Let $\Ii_Q$ be obtained by
extending $\Ii$ with the unique interpretation of the concept names
$A_{\mu,x}$ and $\bar A_{\mu,x}$ faithful to their intended
meaning: if $Q_{\mu, x}$ can be matched in $\Ii$ with $x$
mapped to $d$, then $d \in \left(A_{\mu,x}\right)^{\Ii_Q}$, and otherwise $d \in
\left(\bar A_{\mu,x}\right)^{\Ii_Q}$.
 By construction, $\Ii_Q$ is consistent, and so is each of its
bags. Let $T$ be the set of types realised in $\Ii_Q$. Because $\Ii
\notmodels Q$, no type from $T$ contains both $A_{\tra,x}$
and $A_{\ntra,x}$, which gives the first condition in
Definition~\ref{def:counter-witness}. The root bag of $\Ii_Q$
witnesses the second condition. As each element of $\Ii_Q$
belongs to a $\tra$ bag and a $\ntra$ bag, each  $\tau\in T$ is
realised in some $\tra$ bag and in some $\ntra$ bag. These bags
witness the third condition.

Conversely, let $T \subseteq \Tp(\Kk_Q)$ be a counter-witness for
$Q$. Let $\Ii_0$ be the interpretation guaranteed by the second
condition, and let $\Ii_{\mu, \tau}$ be interpretations guaranteed by
the third condition. From them we build a $\SIF$-tree counter-model
for $Q$ in a top-down fashion. The root bag is $\Ii_0$. Take an
element $d$ that so far only belongs to a $\mu$-bag. By construction,
the type $\tau$ of $d$ belongs to $T$. Let $\nu\neq \mu$. We add to
the $\SIF$-tree under construction a copy of $\Ii_{\mu, \tau}$, with
one element of type $\tau$ replaced by $d$. Because $\Kk$ is
normalised, the resulting $\SIF$-tree $\Ii$ is a model of $\Kk$. The
tree partition of $\Ii$ has finite width because each bag is a copy of
one of the finitely many finite interpretations $\Ii_0$ and $\Ii_{\mu,
\tau}$.

It remains to see that $\Ii \notmodels Q$. 
We first prove by induction on the size of $Q_{\mu,x}$ that for each
homomorphism $f : Q_{\mu,x} \to \Ii$, it holds that $f(x) \in
A_{\mu,x}^{\,\Ii}$. Let $Z_{x}$ and $Z_{f(x)}$ be the
$\mu$-bags of $x$ and $f(x)$, respectively. By the inductive
assumption, $f(y) \in A_{\nu,y}^{\,\Ii}$ for all $y \in Z_{x}
\setminus \{x\}$ and $\nu\neq \mu$. Because $Z_{f(x)}$ 
matches only consistent specialisations, there is a
consistent specialisation $\widetilde Z_x$ of $Z_{x}$ such that $f$
induces a homomorphism from $\widetilde Z_x$ to
$Z_{f(x)}$. From the consistency of $\widetilde Z_x$ it follows that
$f(x) \in A_{\mu,x}^{\,\Ii}$. Now, if $\Ii \models Q$, then there is a
homomorphism $f: Q_i \to \Ii$ for some $i$. Then, $f(x) \in
A^{\,\Ii}_{\tra,x} \cap A^{\,\Ii} _{\ntra,x}$ for all $x
\in\var(Q_i)$. Because all types realised in $\Ii$ occur in $T$, this
contradicts Definition~\ref{def:counter-witness}.
\end{proof}

\begin{theorem} \label{thm:SIF-withprovisos}
The finite query entailment problem for $\SIF$ is in \twoexp.
\end{theorem}


\begin{proof} 
Let $\Kk$ be a \SIF KB using only $\tra$ or only $\ntra$ roles in the
ABox and let $Q$ be a union of connected CQs, each of size at most
$m$. By Lemmas~\ref{lem:SIF-trees}-\ref{lem:sif-characterization},
testing $\Kk \fentails Q$ amounts to deciding if there exists a
counter-witness for $Q$, which can be done using the following variant
of type elimination \cite{Pratt79,abs-1202-0914}. Let $T_0$ be the set
of types from $\Tp(\Kk_Q)$ that contain either $\bar A_{\tra,x}$ or $\bar 
A_{\ntra,x}$ for all $x \in \bigcup_i \var(Q_i)$. For $T\subseteq
T_0$, let $F(T)$ be the set of types $\tau \in T_0$ such that for all
$\mu\in \{\tra,\ntra\}$ there exists a consistent finite model of
the TBox of $\Kk_Q \!\upharpoonright\! \mu$ that realises type $\tau$ and
realises only types from $T$. Then, a set $T$ is a counter-witness if
it is a fixed point of the operator $F$ and satisfies the second
condition of Definition~\ref{def:counter-witness}.  Notice that $F$ is a
monotone operator on subsets of $T_0$. Consequently, $F$ has the
greatest fixed point and it can be obtained by iterating $F$ on
$T_0$: \[T_0 \supseteq F(T_0) \supseteq F^2(T_0) \supseteq \dots
\supseteq F^i (T_0) = F^{i+1} (T_0) \] for some $i\leq |T_0|$. Thus, a
counter-witness for $Q$ exists iff $F^i(T_0)$ satisfies the second
condition of Definition~\ref{def:counter-witness}. It remains to see
how to test this condition and how to compute $F(T)$ for a given
$T$. Both these tasks reduce to finite query entailment modulo types
for simpler logics.

A given $T$ satisfies the second condition of
Definition~\ref{def:counter-witness} iff it is not the case that
$\Kk_Q \! \upharpoonright\! \mu_0 \fentails^T Q'$, where the UCQ $Q'$
is the union of all inconsistent specialisations of the bags of
queries $Q_1, Q_2, \dots, Q_p$ \eqref{eq:queries}.
%
%
The size of $\Kk_Q \! \upharpoonright\! \mu_0$ is bounded by the size
of $\Kk_Q$ which is $|\Kk| + \Oo(mp)$, and $Q'$ is a union of at most
$p\cdot 2^{2m}$ CQs of size $\Oo(m)$.


If $\mu_0=\ntra$, then $\Kk_Q \! \upharpoonright\! \mu_0$ is an $\ALCIF$
KB. By Theorem~\ref{thm:ALCIF}, we can decide if $\Kk_Q \!
\upharpoonright\! \mu_0 \fentails^{T} Q'$ in 
time $2^{\Oo(|\Kk_Q  \upharpoonright \mu_0|+|Q'| \cdot m^m)}$,
which is $2^{\Oo(|\Kk| + mp \cdot  2^{\textrm{poly}(m)})}$.


If $\mu_0=\tra$, then $\Kk_Q \! \upharpoonright\! \mu_0$ is a $\SOI$ KB 
(with no nominals used). Using our previous results on \SOI, we can decide if
$\Kk_Q \! \upharpoonright\! \mu_0 \fentails Q'$ in time $2^{|Q'| \cdot
  |\Kk_Q \upharpoonright \mu_0|^{\Oo(m)}}$, which is 
$2^{mp \cdot (|\Kk| + mp)^{\Oo(m)}}$. We can easily incorporate the set of types $T$
without increasing the complexity: if the ABox contains some type not
in $T$ the algorithm immediately accepts; otherwise, the automaton is
constructed like before, except that the set of all
types is replaced everywhere with $T$. 

To compute $F(T)$ for a given $T$ we need to test for each $\tau \in
T$ and $\mu \in \{\tra, \ntra\}$ whether there is a consistent finite
model of the TBox of $\Kk_Q\! \upharpoonright \! \mu$ that realises
type $\tau$ and realises only types from $T$. For each $\tau$ and
$\mu$ this test can be done just like above, except that in $\Kk_Q\!
\upharpoonright \! \mu$ we replace the ABox with $\{ A(b) \mid A\in
\tau \}$ where $b$ is a fresh individual name. The complexity bounds
for a single test carry over.  To compute the fixed point we need at
most $2^{2mp + |\Kk|}$ iterations of $F$, each requiring at most
$2^{2mp + |\Kk|}$ $\SOI$ tests and at most $2^{2mp + |\Kk|}$ $\ALCIF$
tests. These factors are absorbed by the asymptotic bounds on the cost
of single tests.
Substituting the bound $p\leq |Q|\cdot m^m$ we obtain the bound 
$2^{(|\Kk| + |Q|)^{\mathrm{poly}(m)}}$  for the total running time.

Let us now lift the provisos. Take an arbitrary $\SIF$ KB $\Kk$ and
arbitrary UCQ $Q$. Like for $\SOI$, we can assume that each individual
has its  unary type fully specified in the ABox.
Consider two KBs $\Kk_1$ and $\Kk_2$ obtained from $\Kk$ by removing
from the ABox of $\Kk$ all transitive and all non-transitive
roles, respectively. One can prove (see appendix) that
\mbox{$\Kk \notfentails Q$} iff 
there exist finite interpretations $\Ff_1\models\Kk_1$ and
$\Ff_2\models \Kk_2$ such that for each disjunct $P$ of $Q$, for each
$V \subseteq \var(P)$, for each function $h: V \to \Ind(\Kk)$, for
each partition of the atoms of $P$ into $P_1$ and $P_2$ with
$\var(P_1) \cap \var(P_2) \subseteq V$, for some $i$ it holds that
$\Ff_i \notmodels h(P_i)$, where $h(P_i)$ is a CQ 
with constants obtained from $P_i$ by applying $h$ to variables in
$V$.  
For each $P$, $V$, $h$ and each partition $P_1$, $P_2$ of $P$, guess
whether it is $h(P_1)$ or $h(P_2)$ that will not hold. Let $Q_i$ be
the union of all chosen $h(P_i)$; note that this is a union of
exponentially many CQs of size bounded by the maximal size of $Q$'s
CQs. (The number of possible $Q_i$ is doubly exponential, so eliminating this
nondeterminism adds a doubly exponential factor to the running time.)
It holds that $\Kk \notfentails Q$ iff $\Kk_i\notfentails Q_i$ for all
$i$, and each $\Kk_i$ respects the proviso. 
For the second proviso, consider $R = R_1 \cup \dots \cup R_p$ with
$R_j = R^1_j\land \dots \land R^{q_j}_j$, where $R_j^k$ are connected
CQs over disjoint sets of variables and constants. Then for any KB $\Ll$, $\Ll
\notfentails R$ iff $\Ll \notfentails R_1^{k_1} \cup \dots \cup
R_p^{k_p}$ for some $k_1, \dots, k_p$. The number of sequences $k_1,
\dots, k_p$ to check is singly exponential in $p$. Applying this
construction to $\Kk_i$ and $Q_i$, we arrive at the case where both
provisos are satisfied.  Because $Q_i$ is an exponential union of CQs,
this step introduces a doubly exponential factor to the running time,
but the size bounds for the involved KBs and UCQs are not affected.
After eliminating constants from $Q_i$ in the usual way, we can use
the algorithm described above.  
\end{proof}




\section{Conclusions and Discussion}
We have established  decidability of finite query entailment of \SOI, \SOF and \SIF, and proved that the combined complexity coincides with
that of unrestricted query entailment (\twoexp-complete in all cases). Decidability of  finite query entailment for \SOIF remains open.

Since existing \twoexp-hardness proofs hold for finite query answering for both \ALCI and \ALCO, our upper bound is tight for all logics containing either of these. For ${\cal SF}$ and its fragments,  the best known lower bound is co-\nexptime of  query answering in \Smc~\cite{EiterLOS09}.

One crucial aspect in our techniques is the ability to define a suitable notion of decomposition of counter-models. This appears to be more challenging for logics with role inclusions, and  we conjecture that for fragments of \SOIF extended with role inclusions a different approach is needed. 
A promising direction for future work is to push our techniques to establish tight  bounds for Horn fragments of \SOIF. 


\paragraph*{Acknowledgements.} This work was done in the course of
several meetings in Vienna and Warsaw, made possible by a grant from 
the Austrian Agency for International Cooperation in Education and
Research (OeAD-GmbH) awarded for the project \emph{Logic-based Methods  
in Data Management and Knowledge Representation} within the call
\emph{WTZ Poland 2017-19}. The whole project and this work in particular is
a result of three great brainstorming sessions with Claire David, Magdalena
Ortiz, and Mantas Simkus back in 2016. The first author was
supported by Poland's National Science Centre grant 2016/21/D/ST6/01485. 
Last but not least, we salute the anonymous reviewers of KR 2018 for their hard work.


\newpage

\appendix


\section{Proof of Lemma~\ref{lem:colouring}}

Let $n\geq0$. Because $\Ii\setminus \Nomi(\Kk)$ has bounded degree,
$2n$-neighbourhoods in $\Ii$ have size bounded by some $m$. We colour
the elements of $\Ii$ one by one, with $m$ colours. Pick an uncoloured
element $d$. At most $m-1$ colours are already used in
$N_{2n}^{\Ii}(d)$. Assign to $d$ any colour that is not yet used in
$N_{2n}^{\Ii} (d)$.  This procedure gives an $n$-proper
colouring. Indeed, consider different $e$, $e'$ from $N_n ^{\Ii} (d)$
for some $d\in \Ii$. Without loss of generality we can assume that $e$
was coloured before $e'$. But $e$ belongs to $N_{2n}^{\Ii} (e')$, so
the colours of $e$ and $e'$ are different by construction.  

\section{Proof of Theorem~\ref{thm:SOI-regular}}

To make clique-forests accessible to automata, we encode them as
finitely labelled forests. 
%
Let $\mathsf{TRol}(\Kk)$ be the set of transitive roles from
$\Rol(\Kk)$, and let $[X]^{\leq k}$ be the family of subsets of  $X$
of size at most $k$. In the encoding, nodes are labelled with elements
of the alphabet 
\[ \Sigma = \Tp(\Kk) \cup \left (\mathsf{TRol}(\Kk) \times
    [\Tp(\Kk)]^{\leq |\Kk|}\right)\] and edges are labelled with
elements of the alphabet \[ \Gamma = \Tp(\Kk) \times \Rol(\Kk)\times
  \Tp(\Kk)\,.\]
To produce the encoding of a clique-forest for $\Ii \setminus
\Nomi(\Kk)$. We  order its trees in such a way that the root of the
$i$th tree is the $i$th element of $\Ind(\Kk) \setminus \Nomi(\Kk)$
wrt. some fixed ordering. Then, we label each element node with the
single unary type it realises, and each clique node with its single
nonempty role and the set of unary types it realises. Finally, if in
$\Ii$ there is an $r$-edge from an element of type $\tau$ in some
parent node to an element of type $\sigma$ in some child node, then in
the encoding the edge from the parent node to the child node is
labelled with $(\tau, r, \sigma)$. Because unary types do not repeat
within cliques, this uniquely determines the endpoints. We do not
represent nominals explicitly in the encoding, but thanks to the
initial preprocessing, all relevant information about them is
contained in the unary types of the remaining elements. 


Thus, our automata run over forests built of at most $N = |\Kk|^2$ trees,
with branching bounded by $N$, nodes labelled with elements
of alphabet $\Sigma$
and edges are labelled with elements of the alphabet $\Gamma$.
In such automata, transition relation has the form
\[\delta \subseteq Q \times \Sigma \times (\Gamma \times Q)^{\leq
    N}\,,\] where $Q$ is the set of states.
The automata process the forests top down. The initial states are
specified for each tree separately: the automaton has a set $I
\subseteq Q^{\leq N}$ of sequences of initial states. A run is a
labelling of the input forest with states in such a way 
that the sequence of states in the roots belongs to
$I$, and if a node has state $q$, label $\alpha$, and its children are
connected via edges with labels $\beta_1, \beta_2, \dots, \beta_n$ and have states
$q_1, q_2, \dots, q_n$, then 
\[ (q, \alpha, (\beta_1, q_1), \dots, (\beta_n, q_n)) \in \delta\,.\]
We use \emph{B\"uchi acceptance condition}: we specify a set $F
\subseteq Q$ of marked states that need to be revisited, and consider
a run accepting if on each branch marked states occur infinitely
often. A forest is \emph{accepted} by the automaton if there exists an
accepting run over it.

An automaton has \emph{trivial acceptance condition} if $F=Q$. Then,
each run is accepting but the automaton may still reject some forests,
because there may be no run for them: a branch of the computation can
get stuck if no transition is consistent with the current state, label
and edge labels.
An automaton is \emph{weak} if on each branch of each run, once a
marked state is visited, all subsequent states are marked. Notice that
all automata with trivial acceptance condition are weak.
Given a weak automaton and an arbitrary B\"uchi automaton it is
particularly easy to construct an automaton recognising trees accepted
by both input automata: it suffices to take the standard (synchronous)
product automaton and mark all states that contain a marked states on
both coordinates.

The automaton recognising safe counter-examples for $Q$ is obtained as
a product of automata verifying independently various parts of the
condition.

The first thing to check is the consistency of the encoding: if an
edge has label $(\tau, r, \sigma)$, then $\tau$ must occur in the
label of the parent node, and $\sigma$ must occur in the label of the
child node. To check this, it suffices to examine for each node the
labels of all edges incident to it plus the label of the node
itself. When a transition is made, all these are available except the
label on the edge to the parent: it must be stored in the state. The
automaton has $\Oo(|\Gamma|) = 2^{\Oo(|\Kk|)}$ states and trivial
acceptance condition.

The second thing to check is that the $\SOI$-forest is a model of
$\Kk^*$. Checking that the $\SOI$-forest is a model of the ABox $\Aa$
of $\Kk^*$ amounts to testing if the roots of the trees are labelled with
appropriate types. This can be done easily by an automaton with
$\Oo(|\Kk|)$ states and trivial acceptance condition. To verify
that the TBox is satisfied we need to check each CI. For CIs of the
form \[\bigsqcap_i A_i \sqsubseteq\bigsqcup_j B_j\] we have a
two-state automaton with trivial acceptance condition that simply
tests that each type used in the encoding satisfies this CI; if the type
of some $a\in\Nomi(\Kk)$ specified in $\Aa$ violates this CI,
the automaton rejects everything. 
CIs of the form \[ A \sqsubseteq  \forall  r. B\]  are also easy to
handle. If $\Aa$ contains $A(a)$, $A_{r,b}(a)$, $\bar B(b)$ for some $a$ and $b$,
the automaton rejects everything. Otherwise, it suffices to check that
in the input $\SOI$-forest there is no $r$-edge 
from an element whose unary type contains $A$ to an element whose
unary type contains $\bar B$. This amounts to verifying that none of
the following are used in the encoding:
\begin{itemize}
\item node labels $(r, T)$ such that $A\in \tau \in
  T$ and $\bar B \in \sigma \in T$ for some $\tau$, $\sigma$;
\item edge labels $(\tau, r, \sigma)$
  with $A\in\tau$ and $\bar B \in \sigma$;
\item edge labels $(\sigma, r^-, \tau)$
  with $A\in\tau$ and $\bar B \in \sigma$;
\item unary types containing both $A$ and
  $A_{r,b}$ for some $b$ such that $\bar B(b)\in\Aa$;
\item unary type containing both $\bar B$ and
  $A_{r^-,b}$ for some $b$ such that $A(b)\in\Aa$.
\end{itemize}
These conditions simply disallow certain labels; they can be checked
by a two-state automaton with trivial acceptance condition.

Finally, let us take a CI of the form \[ A \sqsubseteq \exists
r. B\,.\] For ordinary elements this condition can be tested in a
similar way as above, except that one needs access to the label of the
current node and all edges incident to it. Like for the initial
consistency check, it suffices to store in the state the label of the
edge to the parent.  Nominals have to be treated separately, because
they are not explicitly represented in the tree: for each $a$ such
that $A(a) \in \Aa$ and there is no $b$ such that $A_{r,b}(a)\in\Aa$
and $B(b)\in\Aa$, we have a two-state weak automaton looking for a
label that uses a type $\tau$ such that $B\in\tau$ and $A_{r^-,
a}\in\tau$. Note that this automaton has a non-trivial acceptance
condition, but it is weak: as soon as it finds an appropriate label,
it loops in a marked state. Summing up, the total size of the
state-space of the KB component is $2^{\Oo(|\Kk|^2)}$.

The third thing to check is that the query $Q$ is not satisfied. We
begin by replacing query $Q$ with a query $Q'$ such that $\Ii^*
\models Q$ iff $(\Ii \setminus \Nomi(\Kk))^* \models Q$ for each model
$\Ii$ of $\Kk$. The query $Q'$ is obtained in two steps. In the first
step, for each CQ $P$ constituting $Q$, we add to $Q$ each CQ that can
be obtained from $P$ by subdividing some transitive atoms; that is, by
replacing some atoms of the form $r(x,y)$ for some transitive $r$,
with two atoms $r(x,z)$ and $r(z,y)$ for a fresh variable $z$. In the
second step, for each CQ $P$ of the modified $Q$, we add to $Q$ each
CQ that can be obtained from $P$ by performing the following operation
any number of times. Let $\tp(x)$ be the set of all $A$ such that $P$
contains $A(x)$. Choose $x \in\var(P)$ and $a\in \Nomi(\Kk)$ such
that $A(a)\in\Aa$ whenever $A \in \tp(x)$. Drop all atoms of the
form $A(x)$ from $P$. Replace in $P$ each atom of the form $r(y,x)$ by
$A_{r,a}(y)$ and each atom of the form $r(x,y)$ by $A_{r^-,a}(y)$. It
is easy to see that the resulting query $Q'$ has the desired
property. After the first step, the number of CQs grows by the factor
$2^m$ and their size is at most $2m$. After the second step, the
number of CQs grows by the factor $|\Kk|^{2m}$ and their size is still
at most $2m$. Thus, the size of the resulting query is at most
$|Q|\cdot 2^m\cdot |\Kk|^{2m}$, and its CQs have size at most
$2m$. 

Thus, it suffices to construct for each CQ $P$ of $Q'$ an automaton
that tests if $(\Ii \setminus \Nomi(\Kk))^* \models P$ where $\Ii$ is
the $\SOI$-forest represented by the input encoding.  Its states are
composed from an edge label $\beta = (\sigma, r, \tau)$ and a set of
partial functions
\[f:\var(P) \multimap\!\to \{\mathsf{succ}, \mathsf{other}\}\,,\]
representing all partial matchings of $P$ in the interpretation $(\Ii
\setminus \Nomi(\Kk))^*$ restricted to elements represented in the
subtree rooted at the current node. 
The label $\beta=(\sigma, r, \tau)$ is always the label on the edge
from the parent to the current node (if the current node is the root of
the input tree, $\beta$ is arbitrary). Under this assumption there is
a unique element of type $\tau$ in the current node. We refer to
this element as the \emph{current element}. Similarly, in the parent
node there is exactly one element of type $\sigma$; we call it the
\emph{parent element}. In $(\Ii \setminus \Nomi(\Kk))^*$ these two
elements are connected by an $r$-edge.
The identifier $\mathsf{succ}$ stands for any element (represented in
the current subtree) that is an $r$-successor of the parent element in
$ (\Ii \setminus \Nomi(\Kk))^*$. If $r$ is non-transitive, this simply
means the current element. If $r$ is transitive, it means any element
$r$-reachable from the current element. The identifier $\mathsf{other}$
stands for any other element represented in the current subtree.
All states are initial. Transitions are defined only for states that
contain only functions that are not total, and the acceptance
condition is trivial. It is clear that such an automaton is correct
provided that the transition relation ensures the intended semantics
of the states. Let us see how to define it. 

First, we describe when 
\[ \bigg( \big((\sigma,r, \tau), \Phi\big), \tau, \big( (\tau, r_i,
  \tau_i),  ((\tau, r_i, \tau_i), \Phi_i)\big)^n_{i=1} \bigg)\]
is a transition of the automaton. Let $\Psi$ be the set of all
constant partial functions $h : \var(P) \multimap\!\to
\{\tau\}$ such that $\tp(x) \subseteq h(x)$ for all $x \in
\dom(h)$.  
We say that functions $h\in\Psi$ and $f_1 \in\Phi_1, f_2 \in\Phi_2, \dots, f_n\in\Phi_n$ are
\emph{compatible} if they have disjoint domains and for each atom
$s(x, y)$ of $P$
\begin{itemize}
\item if $x\in\dom(h)$, $y\in\dom(f_i)$ then $r_i=s$,
  $f_i(y)=\mathsf{succ}$; 
\item if $y\in\dom(h)$, $x\in\dom(f_i)$ then $r_i=s^-$,
  $f_i(x)=\mathsf{succ}$; 
\item if $x\in\dom(f_i)$, $y\in\dom(f_j)$, $i\neq j$ then $s$ is
  transitive, $r_i=r_j^-=s$ , and $f_i(x)=f_j(y)=\mathsf{succ}$. 
\footnote{
Due to the initial preprocessing of $Q$, this condition is actually
redundant, but we include it to make the correctness more apparent.
}   
\end{itemize}
If $r$ is non-transitive, the condition that each transition of the
form above has to satisfy is that $\Phi$ is the set of all functions
$f$ that can be obtained from any compatible functions $h\in\Psi$ and
$f_1 \in\Phi_1, f_2 \in\Phi_2, \dots, f_n\in\Phi_n$  by setting 
\[f(x)=\begin{cases} 
\mathsf{succ} & \text{if } h(x) = \tau\\
\mathsf{other} & \text{if } f_i(x) = \mathsf{succ} \\
& \text{or } f_i(x) = \mathsf{other}
\end{cases}\]
and if $r$ is transitive, we set 
\[f(x)=\begin{cases} 
\mathsf{succ} & \text{if } h(x) = \tau \\
& \text{or } f_i(x) = \mathsf{succ},  r_i=r\\
\mathsf{other} & \text{if } f_i(x) = \mathsf{other} \\
& \text{or } f_i(x) = \mathsf{succ}, r_i\neq r 
\end{cases}\]
with the convention that whenever we write $g(x)=\gamma$ for a partial
function $g$, we implicitly assume that $x\in\dom(g)$.

For transitions of the form 
\[ \bigg( \big((\sigma, r, \tau), \Phi\big), (r', T),  \big((\sigma_i, r_i,
  \tau_i),  ((\sigma_i, r_i, \tau_i), \Phi_i)\big)^n_{i=1}
  \bigg)\] the condition is similar.
For $\Psi$ be take the set of all partial functions $h : \var(P)
\multimap\!\to T$ such that $\tp(x) \subseteq h(x)$ for all $x \in
\dom(h)$ and the only role atoms in $P$ with both variables in
$\dom(h)$ are $s$ atoms.  
Functions $h\in\Psi$ and $f_1 \in\Phi_1, f_2 \in\Phi_2, \dots,
f_n\in\Phi_n$ are \emph{compatible} if they have disjoint domains and
for each atom $s(x, y)$ of $P$
\begin{itemize}
\item if $x\in\dom(h)$, $y\in\dom(f_i)$ then $r_i=s$,
  $f_i(y)=\mathsf{succ}$, and either $h(x)=\sigma_i$ or $s=r'$ and $s$
  is transitive; 
\item if $y\in\dom(h)$, $x\in\dom(f_i)$ then $r_i=s^-$,
  $f_i(x)=\mathsf{succ}$, and either $h(y)=\sigma_i$ or $s=r'$ and $s$
  is transitive; 
\item if $x\in\dom(f_i)$, $y\in\dom(f_j)$, $i\neq j$ then $s$ is
  transitive, $r_i=r_j^-=s$ , $f_i(x)=f_j(y)=\mathsf{succ}$, and
  either $\sigma_i=\sigma_j$ or $s=r'$. 
\end{itemize}
If $r$ is non-transitive, $\Phi$ is the set of all partial
functions $f$ that can be obtained from any compatible functions $h\in\Psi$ and
$f_1 \in\Phi_1, f_2 \in\Phi_2, \dots, f_n\in\Phi_n$  by setting 
\[f(x) = \begin{cases} 
\mathsf{succ} & \text{if } h(x)=\tau \\
\mathsf{other} & \text{if } h(x)=\tau'\neq\tau \\
&\text{or } f_i(x)= \mathsf{other}
\end{cases}\]
and if $r$ is transitive, we set 
\[f(x)=\begin{cases} 
\mathsf{succ} & \text{if } h(x) = \tau \\
& \text{or } f_i(x) = \mathsf{succ},  r_i=r, \sigma_i=\tau\\
& \text{or } f_i(x) = \mathsf{succ},  r_i=r, r'=r\\
\mathsf{other} & \text{if } f_i(x) = \mathsf{other} \\
& \text{or } f_i(x) = \mathsf{succ}, r_i\neq r \\
& \text{or } f_i(x) = \mathsf{succ}, \sigma_i\neq\tau, r'\neq r\\
\end{cases}\;.\]
To see that this transition relation ensures the intended semantics of
the states one needs to argue that each partial matching is accurately
represented. This can be done by induction on the size of image. For
size one, the matching will be accounted for based solely on the
labels. For larger images, use the inductive hypothesis for
restrictions of the match to variables mapped to the trees rooted at
the children of the current node. 

The total size of the state-space of the query component is $2^{\Oo(3^{2m} \cdot
|Q|\cdot 2^m\cdot |\Kk|^{2m} )} = 2^{|Q|\cdot |\Kk|^{\Oo(m)}}$.

The last component of the automaton checks that the
$\SOI$-forest is safe. Observe that it is 
unsafe if in the input forest there is a branch with consecutive node
and edge labels $\alpha_1 \beta_1 \alpha_2 \beta_2 \dots$ such that
for some transitive $r$ and all $i$ large enough, $\beta_i = (\tau_i, r,
\sigma_i)$ and either $\sigma_i = \tau_{i+1}$ (edges are incident in
the $\SOI$-tree) or $\alpha_{i+i} = (r, T_i)$ (edges are incident with
an $r$-clique). An automaton can easily check that there is no such
branch. Each time it sees a transitive role it moves to an unmarked
state, storing the role. It moves to a marked state as soon as the
condition above is broken. The automaton has $\Oo(|\Kk|)$ states.

The automaton recognising safe counter-examples can be obtained from
these components by the simple product construction described above,
because only the last component is not weak. The resulting product
automaton has $2^{|Q|\cdot|\Kk|^{\Oo(m)}}$ states. An automaton with
$k$ states has total size $\Oo( k \cdot |\Sigma| \cdot (k \cdot
|\Gamma|) ^N + k^N),$ which in our case is $\Oo(k^{|\Kk|^2}\cdot
2^{2\cdot|\Kk|^3+ |\Kk|^2\log|\Kk| + |\Kk|^2 + \log|\Kk|})$. Thus, the
size of the product automaton is also $2^{|Q|\cdot|\Kk|^{\Oo(m)}}$.

\section{Proof of Lemma~\ref{lem:SOI-unrev}}

There exists a natural homomorphism $h: \Ii \to \Mm$, mapping copies
of elements from $\Mm$ to their originals. The homomorphism $h$
induces a homomorphism from $\Ii^*$ to $\Mm$, so $\Ii^*$ cannot
satisfy $Q$. Because $\Ii$ was obtained using a variant of the
standard unravelling procedure, verifying $\Ii \models \Kk^*$ is
routine.
%
Let us see that $\Ii$ is safe. Suppose that $\Ii$ does contain an
infinite simple $r$-path $\pi$ for some transitive role $r$. Because
$\Nomi(\Kk)$ is finite, by skipping a finite prefix we can assume that
$\pi$ never visits $\Nomi(\Kk)$. The image of $\pi$ under the
homomorphism $h$ from the previous paragraph forms an $r$-path
$h(\pi)$ in $\Mm$. Because $\Mm$ is finite, $h(\pi)$ eventually
stabilises in a single strongly connected component $X$ of $r$ in
$\Mm$. By skipping a finite prefix of $\pi$ we can assume that $h(\pi)
\subseteq X$. In the construction of $\Ii$, nominals are only copied
once, so only nominals get mapped to nominals by $h$. 
Consequently, $h(\pi) \subseteq X \setminus \Nomi(\Kk)$.
From the construction of $\Ii$ it further follows that by skipping
another finite prefix we can assume that the first element of $\pi$
belongs to an $r$-clique $X_0$ that contains a representant of each
$C\in\CN(\Kk)$ with a representant in $X \setminus
\Nomi(\Kk)$. Because $X_0$ is finite and $\pi$ is infinite and simple,
$\pi$ eventually leaves $X_0$. Let $d$ be the first element of $\pi$
outside of $X_0$. There exists $C\in\CN(\Kk)$ such that $d\in C^\Ii$
but $C^\Ii\cap X_0 = \emptyset$, for otherwise there would be no
reason to add $d$ to $\Ii$. On the other hand, $h(d) \in C^\Mm
\cap\big( X \setminus \Nomi(\Kk)\big)$, which implies $C^\Ii \cap X_0
\neq \emptyset$ and gives a contradiction. 





\section{Adaptation of the argument for $\SOI$ to $\SOF$}

The argument for $\SOF$ is almost identical to the one for $\SOI$;
differences are few and easy to delimit. All constructions remain the
same, but each time we check that some interpretation
is a model of $\Kk$, we need to verify the functionality
declarations. These are generally ensured by the absence of inverses
in CIs. We list all necessary modifications below.

\begin{enumerate}
\item $\SOF$-forests are defined exactly like $\SOI$-forests. Because
  $\Rol(\Kk)$ contains no inverses, all edges between instances $\Ii_v$ point
  down the tree.  
\item In the construction of the automaton from
  Theorem~\ref{thm:SOI-regular} we include an additional
  component for each functionality declaration $\mn{Fn}(r)$. To check
  functionality of $r$ for ordinary nodes it suffices to examine the
  label of the node and the labels on all incident edges, which only
  requires storing in the state the label of the edge to the
  parent. Additionally, for all $a\in\Nomi(\Kk)$, if the ABox contains $A_{r,b}(a)$ and
  $A_{r,b'}(a)$  for some $b\neq b'$, the automaton trivially rejects
  everything; if the ABox contains $A_{r,b}(a)$ for only one $b$, the automaton
  checks that no type used in the input forest contains $A_{r^-,a}$; if
  the ABox contains no $A_{r,b}(a)$, the automaton checks that a type
  with $A_{r^-,a}$ occurs at most once in the input forest. The total
  number of states in the described component is
  $2^{\Oo(|\Kk|^2)}$, so including it does not affect the overall
  upper bound.

\item Checking that the unravelling procedure produces a model of
  $\Kk$ (Lemma~\ref{lem:SOI-unrev}) requires verifying the
  functionality declarations. This is routine.
 \item After $\Ff_n$ has been constructed from a $\SOF$-forest using the
   coloured blocking principle, we need to check that it satisfies all
   functionality declarations of $\Kk$. This follows immediately from
   the fact that each redirected edge is a forward edge. 
\end{enumerate}


\section{Proof of Theorem~\ref{thm:ALCIF}}

Each $\ALCIF$ KB can be expressed in the guarded fragment with two
variables and counting ($\mathcal{GC}^2$). Hence, the the following result is
relevant for us.

\begin{theorem}{[\cite{Pratt-Hartmann_datacomplexity}, Theorem~4]}
\label{gc2_qa} 
For any $\mathcal{GC}^2$-sentence $\phi$ and any positive conjunctive query
$\psi$ both finite and infinite query entailment are in
\textsc{co-NP} in terms of data complexity.  
\end{theorem}

Because we are interested in combined complexity and UCQs,
we have to inspect the proof rather than just using the theorem as
a black box. 

The first step of the proof is to show that if $\phi$
(together with some ground atoms) entails $\psi$ then $\phi$ entails a
treeification of $\psi$, which can be rewritten as a $\mathcal{GC}^2$ formula
$\psi_{\mathcal{GC}^2}$. It is easy to see that the same argument applies to
UCQs. For a single CQ $\psi$ there are at most $|\psi|^{|\psi|}$
possible treeifications, therefore for  our UCQ $Q$ we have at most
$n\cdot m^m$ possible treeifications. 

The next step is to use finite query answering for $\mathcal{GC}^2$.

\begin{theorem}{[\cite{Pratt-Hartmann0_gc2_sat}, Theorem~1]
\label{gc2_sat}}
  Finite satisfiability for $\mathcal{GC}^2$ is in EXPTIME.
\end{theorem}

Once again, to obtain the precise bounds we need a bit more than the
stated theorem provides. The proof of Theorem~\ref{gc2_sat} uses the
well-known technique of inequality systems, developed by
Pratt-Hartmann. The provided algorithm is polynomial in the size of
the formula and exponential in the signature under the assumption that
the formula is in the normal form. The normalisation of an arbitrary
formula $\phi$ increases the size of the signature by $\Oo(|\phi|)$,
and that we can afford.

In the inequality system from the proof of Theorem~\ref{gc2_sat}, each
variable represents a star type realised in a hypothetical
counter-model. The star type of an element is a refinement of its
unary type, so we can for free incorporate into the proof the
restriction on allowed unary types: we simply remove variables whose
associated unary type is not in $T$. This procedure does not
complicate the inequality system in any measure.
 
\section{Proof of the claim in the proof of Theorem~\ref{thm:SIF-withprovisos}}

The claim can be equivalently formulated as follows: \mbox{$\Kk \notfentails Q$} iff
there exist finite interpretations $\Ff_1\models\Kk_1$ and
$\Ff_2\models \Kk_2$ such that for each disjunct $P$ of $Q$, for each
$V \subseteq \var(P)$, for each function $h: V \to \Ind(\Kk)$, for
each partition of the atoms of $P$ into $P_1$ and $P_2$ with
$\var(P_1) \cap \var(P_2) \subseteq V$, one cannot simultaneously
extended $h$ to homomorphisms $h_i : P_i \to \Ff_i$ for all
$i$. 

Suppose first that there is a finite
interpretation $\Ff$ such that $\Ff\models \Kk$ and $\Ff \notmodels
Q$. We construct interpretations $\Ff_1$ and $\Ff_2$, as specified in
the claim, by unravelling $\Ff$ like in the proof of
Lemma~\ref{lem:SIF-trees}, but only to a finite depth.  For $\Ff_1$ we
start from $\Ff_\tra$, and for each element $d$ in $\Ff_\tra$ we add a
copy of $\Ff$ with all elements fresh except $d$. For $\Ff_2$ we start
from $\Ff_\ntra$, for each element $d$ in $\Ff_\ntra$ add a copy of
$\Ff_\tra$ with all elements fresh except $d$, and then for each
element $e$ that only belongs to a copy of $\Ff_\tra$, add a copy of
$\Ff$ with all elements fresh except $e$. In both cases, we close the
interpretations of transitive roles under transitivity. Because copies
of $\Ff$ share elements only with copies of $\Ff_\tra$, functionality
requirements do not get violated in the unravelling process. It
follows that $\Ff_i\models \Kk_i$. Consider a disjunct $P$ of $Q$, a
set $V \subseteq \var(P)$, a function $h: V \to \Ind(\Kk)$, and a
partition of $P$ into $P_1$ and $P_2$ such that $\var(P_1) \cap
\var(P_2) \subseteq V$. Suppose that $h$ can be extended to a
homomorphism $h_i : P_i \to \Ff_i$ for all $i$. By the construction of
$\Ff_i$, there exists a homomorphism $f_i: \Ff_i \to
\Ff$. Consequently, we obtain a match for $P$ in $\Ff$ by taking
$f_1\circ h_1 \cup f_2 \circ h_2$, which is a contradiction. Thus,
$\Ff_1$ and $\Ff_2$ are as we wanted.

Conversely, assume that we have finite interpretations $\Ff_1$ and
$\Ff_2$ as in the claim. We first unravel them like above.  To obtain
$\Ff'_1$ we start from $(\Ff_1)_\tra$ and for each $d$ in
$(\Ff_1)_\tra$ we add a copy of $\Ff_1$ with all elements fresh except
$d$. For $\Ff'_2$ we start from $(\Ff_2)_\ntra$, for each $d$ in
$(\Ff_2)_\ntra$ we add a copy of $(\Ff_2)_\tra$ with all elements
fresh except $d$, and then for each $e$ that belongs only to a copy of
$(\Ff_2)_\tra$, add a copy of $\Ff_2$ with all elements fresh except
$e$. Again, close the interpretations of transitive roles under
transitivity. By construction, $\Ff'_1$ and $\Ff'_2$ also satisfy the
condition in the claim. To construct $\Ff$, first delete all subtrees
of the tree partitions of $\Ff'_1$ and $\Ff'_2$ rooted in second-level
nodes that contain an element of $\Ind(\Kk)$, and then take the union
of the two resulting interpretations. This is consistent because all
$a\in \Ind(\Kk)$ have their types fully specified. An argument similar
to the one above shows that $\Ff \models \Kk$ and $\Ff\notmodels Q$.

\end{document}